\documentclass{article} 

\usepackage[colorlinks=true,citecolor=blue]{hyperref}
\usepackage{url}

\usepackage{times}
\usepackage{graphicx} 
\usepackage{subfig} 

\usepackage{natbib}

\usepackage[left=1.25in, top=1in, bottom=1in, right=1.25in]{geometry}

\usepackage{amsmath,amssymb,amsfonts,amsthm}
\usepackage{bm}
\usepackage[mathscr]{eucal}
\usepackage{stmaryrd}
\usepackage{accents}

\usepackage{natbib}

\usepackage{empheq}
\usepackage{paralist}

\usepackage{booktabs,colortbl,rotating}
\usepackage{multirow}
\usepackage{diagbox}
\usepackage{float}

\usepackage{etex}

\usepackage{pgfplots}
\usepackage{tikz}
\usetikzlibrary{arrows,intersections}

\usepackage{algorithm}
\usepackage{algpseudocode}

\graphicspath{{figs//}}

\definecolor{tabgrey}{rgb}{0.8,0.8,0.8}
\definecolor{ForestGreen}{rgb}{0.5,0.9,0}

\newcommand{\troll}[1]{}

\newcommand{\mytt}[1]{{\small\begin{alltt}#1\end{alltt}}}

\newcommand{\eg}{e.g.\ }
\newcommand{\ie}{i.e.\ }

\newcommand{\defEq}{\stackrel{.}{=}}

\newcommand{\indicator}[1]{\llbracket #1 \rrbracket}
\newcommand{\sign}{\operatorname{sign}}

\newcommand{\argmin}[2]{\underset{#1}{\operatorname{Argmin }}\, #2}

\newcommand{\argminUnique}[2]{\underset{#1}{\operatorname{argmin }}\, #2}

\renewcommand{\Pr}{\mathbb{P}}
\newcommand{\Expectation}[2]{\underset{#1}{\mathbb{E}}\left[ #2 \right]}

\newcommand{\X}{\mathsf{X}}
\newcommand{\Y}{\mathsf{Y}}
\newcommand{\SSf}{\mathsf{S}}

\newcommand{\FCal}{\mathscr{F}}
\newcommand{\HCal}{\mathscr{H}}

\newcommand{\RCal}{\mathscr{R}}
\newcommand{\SCal}{\mathscr{S}}

\newcommand{\XCal}{\mathscr{X}}

\newcommand{\Real}{\mathbb{R}}

\newcommand{\PMOne}{\{ \pm 1 \}}

\newcommand{\scorer}{s \colon \XCal \to \Real}

\newcommand{\PosDist}{P}
\newcommand{\NegDist}{Q}

\newcommand{\D}{D}
\newcommand{\DPQ}{\D_{\PosDist, \NegDist, \pi}}
\newcommand{\DMN}{\D_{M, \eta}}

\newcommand{\Contaminator}[1]{\accentset{\rule{.4em}{0.6pt}}{#1}}

\newcommand{\MCont}{\Contaminator{M}}
\newcommand{\etaContam}{\Contaminator{\eta}}

\newcommand{\ContamDShort}{\Contaminator{D}}

\newcommand{\rhoPlus}{\rho}

\newcommand{\meanMapPos}{\mu_{P}}
\newcommand{\meanMapNeg}{\mu_{Q}}

\newcommand{\SLN}{\mathrm{SLN}}
\newcommand{\CCNDSymm}{\SLN( D, \rhoPlus )}

\newcommand{\ellZeroOne}{\ell^{01}}

\newcommand{\unhinged}{\mathrm{unh}}
\newcommand{\hinge}{\mathrm{hinge}}
\newcommand{\whinge}{\mathrm{whinge}}
\renewcommand{\square}{\mathrm{sq}}

\newcommand{\ellSquare}{\ell^{\square}}
\newcommand{\ellHinge}{\ell^{\hinge}}
\newcommand{\ellUnhinged}{\ell^{\unhinged}}
\newcommand{\ellWhinge}{\ell^{\mathrm{whinge}}}

\newcommand{\OptWeight}[1]{w^*_{#1, \lambda}}
\newcommand{\OptScorer}[1]{s^*_{#1, \lambda}}

\newcommand{\ellContam}{\Contaminator{\ell}}

\newcommand{\Risk}{\mathbb{L}}

\newcommand{\EllRisk}[2]{\mathbb{L}^{#1}_{\ell}(#2)}

\newcommand{\regret}{\mathrm{regret}}

\newcommand{\BayesOpt}[2]{\SCal^{#1, *}_{#2}}
\newcommand{\RestrictedOpt}[2]{\SCal^{#1, *}_{#2}}

\newcommand{\LinearClass}{\FCal_{\mathrm{lin}}}
\newcommand{\RegLinearClass}{\FCal_{\mathrm{lin}, \lambda}}
\newcommand{\AllScorers}{\Real^{\XCal}}
\newcommand{\KernelScorers}{\FCal_{\HCal, \lambda}}
\newcommand{\BoundedScorers}{\FCal_{B}}

\newcommand{\AllDists}{\Delta}

\newcommand{\RobustCCNSymm}{\mathcal{N}_{\mathrm{sln}}}

\newtheorem{proposition}{Proposition}

\newtheorem{lemma}[proposition]{Lemma}

\newtheorem{observation}{Observation}

\newtheorem{definition}[observation]{Definition}

\newcommand{\suppRef}[1]{Appendix \ref{#1}}

\newtheorem{conjecture}{Conjecture}

\title{Learning with Symmetric Label Noise: The Importance of Being Unhinged}

\author{
Brendan van Rooyen$^{*, \dagger}$ \\
\and
Aditya Krishna Menon$^{\dagger, *}$ \\
\and
Robert C. Williamson$^{*, \dagger}$ \\
$^*$The Australian National University \qquad $^\dagger$National ICT Australia \\
{\footnotesize \texttt{\{ brendan.vanrooyen, aditya.menon, bob.williamson \}@nicta.com.au} } \\
}

\begin{document}

\maketitle

\begin{abstract} 
Convex potential minimisation is the \emph{de facto} approach to binary classification.
However, \citet{Long:2010} proved that under symmetric label noise (SLN), minimisation of \emph{any} convex potential 
over a linear function class
can result in classification performance equivalent to random guessing. 
This ostensibly shows that convex losses are not SLN-robust.
In this paper, we propose a convex, classification-calibrated loss and prove that it \emph{is} SLN-robust.
The loss avoids the \citet{Long:2010} result by virtue of being \emph{negatively unbounded}. 
The loss is a modification of the hinge loss, where one does not clamp at zero; hence, we call it the \emph{unhinged loss}.
We show that the optimal unhinged solution is equivalent to that of a strongly regularised SVM,
and is the limiting solution for \emph{any} convex potential;
this implies that strong $\ell_2$ regularisation makes most standard learners 
SLN-robust.
Experiments confirm the SLN-robustness of the unhinged loss.
\end{abstract} 

\section{Learning with symmetric label noise}
\label{sec:intro}

Binary classification is the canonical supervised learning problem.
Given an instance space $\XCal$, and samples from some distribution $\D$ over $\XCal \times \PMOne$, the goal is to learn a scorer $\scorer$ with low \emph{misclassification error} on future samples drawn from $\D$.
Our interest is in the more realistic scenario where the learner observes samples from a distribution $\ContamDShort$, which is a corruption of $\D$ where labels have some constant probability of being flipped.
The goal is still to perform well with respect to the unobserved distribution $\D$.
This is known as the problem of learning from {symmetric label noise} ({SLN learning}) \citep{Angluin:1988}.

\citet{Long:2010} proved the following negative result on what is possible in SLN learning: 
there exists a linearly separable $\D$ where, when the learner observes some corruption $\ContamDShort$ with symmetric label noise of \emph{any nonzero rate}, minimisation of \emph{any convex potential} over a linear function class results in classification performance on $\D$ that is equivalent to random guessing.
Ostensibly, this establishes that convex losses are not ``SLN-robust'' 
and motivates the use of non-convex losses \citep{Stempfel:2007b, Hamed:2010, Ding:2010, Denchev:2012, Manwani:2013}.

In this paper, we propose a convex loss and prove that it \emph{is} SLN-robust.
The loss avoids the result of \citet{Long:2010} by virtue of being \emph{negatively unbounded}.
The loss is a modification of the hinge loss where one does not clamp at zero; thus, we call it the \emph{unhinged loss}.
We show that this is the unique convex loss (up to scaling and translation) that satisfies a notion of ``strong SLN-robustness '' (Proposition \ref{prop:everything-is-unhinged}).
In addition to being SLN-robust, this loss has several attractive properties, such as
being classification-calibrated (Proposition \ref{prop:cc}), 
consistent when minimised on the corrupted distribution (Proposition \ref{prop:surrogate-regret}),
and
having an easily computable optimal solution that is the difference of two kernel means (Equation \ref{eqn:mmd}).
Finally, we show that this optimal solution is equivalent to that of a strongly regularised SVM (Proposition \ref{prop:highly-regularised}),
and such a result holds more generally for \emph{any} twice-differentiable convex potential (Proposition \ref{prop:really-everything-is-unhinged}),
implying that strong $\ell_2$ regularisation endows most standard learners 
with SLN-robustness.

The classifier resulting from minimising the unhinged loss is not new \citep[Chapter 10]{Devroye:1996}, \citep[Section 1.2]{Scholkopf:2002}, \citep[Section 5.1]{Shawe-Taylor:2004}.
However, establishing this classifier's SLN-robustness, its equivalence to a highly regularised SVM solution, and showing the underlying loss uniquely satisfies a notion of strong SLN-robustness, to our knowledge \emph{is} novel.

\section{Background and problem setup}
\label{sec:problem-setup}

Fix an instance space $\XCal$.
We denote by $\D$ some distribution over $\XCal \times \PMOne$,
with random variables $(\X, \Y) \sim D$. 
Any $\D$ may be expressed via the \emph{class-conditional distributions}
$(P, Q) = ( \Pr( \X \mid \Y = 1 ), \Pr( \X \mid \Y = -1 ) )$ and \emph{base rate} $\pi = \Pr( \Y = 1 )$, or equivalently
via the \emph{marginal distribution} $M = \Pr( \X )$ and \emph{class-probability function} $\eta \colon x \mapsto \Pr( \Y = 1 \mid \X = x )$. 
We interchangeably write $\D$ as $\DPQ$ or $\DMN$.

\subsection{Classifiers, scorers, and risks}

A \emph{scorer} is any function $\scorer$.
A \emph{loss} is any function $\ell \colon \PMOne \times \Real \to \Real$.
We use $\ell_{-1}, \ell_{1}$ to refer to $\ell(-1, \cdot)$ and $\ell(1, \cdot)$.
The \emph{$\ell$-conditional risk} $L_{\ell} \colon [0, 1] \times \Real \to \Real$ is defined as
$L_{\ell} \colon ( \eta, v ) \mapsto \eta \cdot \ell_1(v) + (1 - \eta) \cdot \ell_{-1}(v).$
Given a distribution $\D$, the \emph{$\ell$-risk} of a scorer $s$ is defined as
\begin{align}
\label{eqn:classification-risk}
\Risk^{\D}_{\ell}( s ) &\defEq \Expectation{(\X, \Y) \sim \D}{\ell( \Y, s(\X) )},
\end{align}
or equivalently $\Risk^{\D}_{\ell}( s ) = \Expectation{\X \sim M}{ L_{\ell}( \eta(\X), s(\X) ) }$.
For a set $\SCal$, $\Risk^{\D}_{\ell}( \SCal )$ is the set of $\ell$-risks for all scorers in $\SCal$.

A \emph{function class} is any $\FCal \subseteq \Real^{\XCal}$.
Given some $\FCal$, the set of \emph{restricted Bayes-optimal scorers} for a loss $\ell$ are those scorers in $\FCal$ that minimise the $\ell$-risk:
$$ \RestrictedOpt{\D, \FCal}{\ell} \defEq \argmin{ s \in \FCal }{ \mathbb{L}^{\D}_{\ell}( s )  }. $$
The set of (unrestricted) Bayes-optimal scorers is $\BayesOpt{\D}{\ell} = \RestrictedOpt{\D, \FCal}{\ell}$ for $\FCal = \Real^{\XCal}$.
The \emph{restricted $\ell$-regret} of a scorer is its excess risk over that of any restricted Bayes-optimal scorer:
$$ \regret^{\D, \FCal}_{\ell}( s ) \defEq \mathbb{L}^{\D}_{\ell}( s ) - \inf_{ t \in \FCal } \mathbb{L}^{\D}_{\ell}( t ). $$

Binary classification is concerned with the risk corresponding to the \emph{zero-one loss}, $\ellZeroOne \colon (y, v) \mapsto \indicator{ y v < 0 } + \frac{1}{2} \indicator{v = 0}$.
A loss $\ell$ is \emph{classification-calibrated} if all its Bayes-optimal scorers are also optimal for zero-one loss:
$ ( \forall \D ) \, \BayesOpt{\D}{\ell} \subseteq \BayesOpt{\D}{01}. $
A \emph{convex potential} is any loss $\ell \colon (y, v) \mapsto \phi(y v)$, where $\phi \colon \Real \to \Real_+$ is convex, non-increasing, differentiable with $\phi'( 0 ) < 0$, and $\phi( +\infty ) = 0$ \citep[Definition 1]{Long:2010}.
All convex potential losses are classification-calibrated \citep[Theorem 2.1]{Bartlett:2006}.


%

%
\subsection{Learning with symmetric label noise (SLN learning)}
\label{sec:special-cases}


The problem of learning with \emph{symmetric label noise} (\emph{SLN learning}) is the following \citep{Angluin:1988, Kearns:1998, Blum:1998, Natarajan:2013}.
For some notional ``clean'' distribution $\D$, which we would like to observe, we instead observe samples from some corrupted distribution $\CCNDSymm$, for some $\rhoPlus \in [0, {1}/{2})$.
The distribution $\CCNDSymm$ is such that the marginal distribution of instances is unchanged, but each label is independently flipped with probability $\rhoPlus$.
The goal is to learn a scorer from these corrupted samples such that $\Risk^{\D}_{01}( s )$ is small.

For any quantity in $\D$, we denote its corrupted counterparts in $\CCNDSymm$ with a bar, \eg $\MCont$ for the corrupted marginal distribution, and $\etaContam$ for the corrupted class-probability function; 
additionally, when $\rhoPlus$ is clear from context, we will occasionally refer to $\CCNDSymm$ by $\ContamDShort$.
By definition of the corruption process, the corruption marginal distribution $\MCont = M$, and \citep[Lemma 7]{Natarajan:2013}
\begin{equation}
\label{eqn:ccn-eta}
(\forall x \in \XCal) \, \etaContam(x) = (1 - 2 \rhoPlus) \cdot \eta(x) + \rhoPlus.
\end{equation}



\section{SLN-robustness: formalisation}


For our purposes, a learner $( \ell, \FCal )$ comprises a loss $\ell$, and a function class $\FCal$, 
with learning being the search for some $s \in \FCal$ that minimises the $\ell$-risk.
Informally, the learner $(\ell, \FCal)$ is ``robust'' to symmetric label noise (SLN-robust) if minimising $\ell$ over $\FCal$ gives the same classifier on both the clean distribution $\D$, which the learner would \emph{like} to observe, and $\CCNDSymm$ for \emph{any $\rhoPlus \in [0, 1/2)$}, which the learner \emph{actually} observes.
We now formalise this notion, and review what is known about the existence of SLN-robust learners. 


%
\subsection{SLN-robust learners: a formal definition}
\label{sec:what-is-robustness}

For some fixed instance space $\XCal$, let $\AllDists$ denote the set of distributions on $\XCal \times \PMOne$.
Given a notional ``clean'' distribution $\D$, $\RobustCCNSymm \colon \AllDists \to 2^\AllDists$ returns the \emph{set} of possible corrupted versions of $\D$ the learner may observe, where labels are flipped with unknown probability $\rho$:
$$ \RobustCCNSymm \colon \D \mapsto \left\{ \CCNDSymm \mid \rhoPlus \in \left[ 0, \frac{1}{2} \right) \right\}. $$
Equipped with this, we define our notion of SLN-robustness.

\begin{definition}[SLN-robustness]
\label{defn:sln-robust}
We say that a learner $( \ell, \FCal )$ is \emph{SLN-robust} if
\begin{equation}
\label{eqn:sln-robust}
( \forall \D \in \AllDists ) \, ( \forall \ContamDShort \in \RobustCCNSymm( \D ) ) \, \Risk^{\D}_{01}( \RestrictedOpt{\D, \FCal}{\ell} ) = \Risk^{\D}_{01}( \RestrictedOpt{\ContamDShort, \FCal}{\ell} ).
\end{equation}
\end{definition}


That is, SLN-robustness requires that
for \emph{any} level of label noise in the observed distribution $\ContamDShort$,
the classification performance (wrt $\D$) of the learner is the same
as if the learner directly observes $\D$. 
Unfortunately, as we will now see, a widely adopted class of learners is \emph{not} SLN-robust.

\subsection{Convex potentials with linear function classes are not SLN-robust}

Fix $\XCal = \Real^d$, and consider learners employing a convex potential $\ell$, and a function class of linear scorers
$$ \LinearClass = \{ x \mapsto \langle w, x \rangle \mid w \in \Real^d \}. $$
This captures \eg the linear SVM and logistic regression, which are 
widely studied in theory and applied in practice.
Unfortunately, these learners are \emph{not} SLN-robust:
\citet[Theorem 2]{Long:2010} give an example where, when learning under symmetric label noise, for \emph{any} convex potential $\ell$, the corrupted $\ell$-risk minimiser over $\LinearClass$ has classification performance equivalent to random guessing on $\D$.
This implies that $( \ell, \LinearClass )$ is not SLN-robust\footnote{Even if we weaken the notion of SLN-robustness to allow for a difference of $\epsilon \in [0, 1/2]$ between the clean and corrupted minimisers' performance, \citet[Theorem 2]{Long:2010} implies that in the worst case $\epsilon = 1/2$.} as per Definition \ref{defn:sln-robust}.
(All Proofs may be found in \suppRef{sec:proofs}.)

\begin{proposition}[{\citet[Theorem 2]{Long:2010}}]
\label{eqn:noise-defeats-all}
Let $\XCal = \Real^d$ for any $d \geq 2$.
Pick any convex potential $\ell$.
Then,
$( \ell, \LinearClass )$ is not SLN-robust.
\end{proposition}

The widespread practical use of convex potential based learners makes  Proposition \ref{eqn:noise-defeats-all} a disheartening result, and motivates the search for other learners that \emph{are} SLN-robust.

\subsection{The fallout: what learners \emph{are} SLN-robust?}

In light of Proposition \ref{eqn:noise-defeats-all}, there are two ways to proceed in order to obtain SLN-robust learners: either we change the class of losses $\ell$, or we change the function class $\FCal$.

The first approach has been pursued in a large body of work that embraces non-convex losses \citep{Stempfel:2007b, Hamed:2010, Ding:2010, Denchev:2012, Manwani:2013}.
However, while such losses avoid the conditions of Proposition \ref{eqn:noise-defeats-all}, this does not automatically imply that they are SLN-robust when used with $\LinearClass$.
In \suppRef{sec:app-robustness-non-convex}, we present evidence that some of these losses are in fact \emph{not} SLN-robust when used with $\LinearClass$.

The second approach is to instead consider suitably rich $\FCal$ that contains the Bayes-optimal scorer for $\ContamDShort$, \eg by employing a universal kernel.
With this choice, one can still use a convex potential loss; in fact, owing to Equation \ref{eqn:ccn-eta}, using \emph{any} classification-calibrated loss will result in an SLN-robust learner when $\FCal = \AllScorers$.

\begin{proposition}
\label{eqn:kernel-defeats-noise}
Pick any classification-calibrated $\ell$.
Then, $( \ell, \AllScorers )$ is SLN-robust.
\end{proposition}

Both approaches have drawbacks.
The first approach has a computational penalty, as it requires optimising a non-convex loss.
The second approach has a statistical penalty, as estimation rates with a rich $\FCal$ will require a larger sample size.
Thus, it appears that SLN-robustness involves a computational-statistical tradeoff.

However, there is a variant of the first option: pick a loss that is convex, \emph{but not a convex potential}.
If an SLN-robust loss of this type exists, it affords the computational and statistical advantages of minimising convex risks with linear scorers.
\citet{Manwani:2013} demonstrated that square loss, $\ell(y, v) = (1 - yv)^2$, is one such loss.
We will show that there is a simpler loss that is similarly convex, classification-calibrated, and SLN-robust, but is not in the class of convex potentials by virtue of being \emph{negatively unbounded}.
To derive this loss, it is helpful to interpret robustness in terms of a noise-correction procedure on loss functions.


\section{SLN-robustness: a noise-corrected loss perspective}
\label{sec:expectations}

The definition of SLN-robustness (Equation \ref{eqn:sln-robust}) involves optimal scorers with the \emph{same loss} $\ell$ over two \emph{different distributions}.
We now re-express this to reason about optimal scorers on the \emph{same distribution}, but with two \emph{different losses}.
This will help characterise the set of losses that are SLN-robust.

\subsection{Reformulating SLN-robustness via noise-corrected losses}

Given any $\rho \in [0, 1/2)$, 
\citet[Lemma 1]{Natarajan:2013} showed how to associate with a loss $\ell$ a \emph{noise-corrected} counterpart $\ellContam$,
such that for any $\D$, $\EllRisk{\D}{s} = \Risk^{\ContamDShort}_{\ellContam}(s)$.
The loss $\ellContam$ is defined as follows.

\begin{definition}[Noise-corrected loss\troll{\footnote{The procedure can also be understood as a MEthod of Loss Transfer; so, we can also think of this as a \,\customfont{MELT}\hspace{-3pt}ed loss.}}]
Given any loss $\ell$ and $\rhoPlus \in [0, 1/2)$, the noise-corrected loss $\ellContam$ is
\begin{equation}
\label{eqn:melted-loss}
( \forall y \in \PMOne ) \, ( \forall v \in \Real ) \, \ellContam( y, v ) =  \frac{ {(1 - \rhoPlus)} \cdot \ell( y, v ) - {\rhoPlus} \cdot \ell( -y, v ) }{1 - 2\rhoPlus}.
\end{equation}
\end{definition}

Since $\ellContam$ depends on the unknown parameter $\rhoPlus$, it is not directly usable to design an SLN-robust learner.
Nonetheless, it is a useful theoretical construct, since 
the risk equivalence between $\EllRisk{\D}{s}$ and $\Risk^{\ContamDShort}_{\ellContam}(s)$ means that for \emph{any} $\FCal$, minimisation of the $\ell$-risk on $\D$ over $\FCal$ is equivalent to minimisation of the $\ellContam$-risk on $\ContamDShort$ over $\FCal$, \ie
$ \RestrictedOpt{\D, \FCal}{\ell} = \RestrictedOpt{\ContamDShort, \FCal}{\ellContam}. $
With this, we can re-express the SLN-robustness of a learner $(\ell, \FCal)$ as
\begin{equation}
\label{eqn:sln-robust-2}
( \forall \D \in \AllDists ) \, ( \forall \ContamDShort \in \RobustCCNSymm( \D ) ) \, \Risk^{\D}_{01}( \RestrictedOpt{\ContamDShort, \FCal}{\ellContam} ) = \Risk^{\D}_{01}( \RestrictedOpt{\ContamDShort, \FCal}{\ell} ).
\end{equation}
This reformulation is useful, because to characterise SLN-robustness of $( \ell, \FCal )$, we can now consider conditions on $\ell$ such that $\ell$ and its noise-corrected counterpart $\ellContam$  induce the same restricted Bayes-optimal scorers.

\subsection{Characterising a stronger notion of SLN-robustness}
\label{sec:order-equivalent}

\citet[Theorem 1]{Manwani:2013} proved a \emph{sufficient} condition on $\ell$ such that Equation \ref{eqn:sln-robust-2} holds, namely,
\begin{equation}
\label{eqn:constant}
( \exists C \in \Real ) ( \forall v \in \Real) \, \ell_1(v) + \ell_{-1}(v) = C.
\end{equation}
For such a loss, $\ellContam$ is a scaled and translated version of $\ell$, so that trivially $\RestrictedOpt{\ContamDShort, \FCal}{\ell} = \RestrictedOpt{\ContamDShort, \FCal}{\ellContam}$. 

Ideally, one would like to \emph{characterise} when Equation \ref{eqn:sln-robust-2} holds.
While this is an open question, 
interestingly, we can show that under a \emph{stronger} requirement on the losses $\ell$ and $\ellContam$, the condition in Equation \ref{eqn:constant} is also \emph{necessary}.
The stronger requirement is that the corresponding risks 
\emph{order all stochastic scorers identically}.
A stochastic scorer is simply a mapping $f \colon \XCal \to \Delta_{\Real}$, where $\Delta_{\Real}$ is the set of distributions over the reals.
In a slight abuse of notation, we denote the $\ell$-stochastic risk of $f$ by
$$ \Risk^{\D}_{\ell}( f ) = \Expectation{ ( \X, \Y ) \sim \D }{ \Expectation{\SSf \sim f( \X )}{\ell( \Y, \SSf )} }. $$
Equipped with this, we define a notion of order equivalence of loss pairs.

\begin{definition}[Order equivalent loss pairs]
We say that a pair of losses $( \ell , \tilde{\ell} )$ are order equivalent if
$$ ( \forall \D ) \, ( \forall f, g \in \Delta_{\Real}^{\XCal} ) \, \Risk^{\D}_{\ell}( f ) \leq \Risk^{\D}_{\ell}( g ) \iff \Risk^{\D}_{\tilde{\ell}}( f ) \leq \Risk^{\D}_{\tilde{\ell}}( g ). $$
\end{definition}

Clearly, if two losses are order equivalent, their corresponding risks have the same restricted minimisers.
Consequently, if $(\ell, \ellContam)$ are order equivalent for every $\rho \in [0, 1/2)$, this implies that $\RestrictedOpt{\ContamDShort, \FCal}{\ell} = \RestrictedOpt{\ContamDShort, \FCal}{\ellContam}$ for \emph{any} $\FCal$,
which by Equation \ref{eqn:sln-robust-2} means that 
for \emph{any} $\FCal$, the learner $( \ell, \FCal )$ is SLN-robust.
We can thus think of order equivalence of $( \ell, \ellContam )$ as signifying \emph{strong SLN-robustness} of a loss $\ell$.

\begin{definition}[Strong SLN-robustness]
We say a loss $\ell$ is strongly SLN-robust if for every $\rhoPlus \in [0, 1/2)$, $(\ell, \ellContam)$ are order equivalent.
\end{definition}

We establish that the sufficient condition of Equation \ref{eqn:constant} is also \emph{necessary} for strong SLN-robustness of $\ell$.

\begin{proposition}
\label{prop:eigen-MELT}
A loss $\ell$ is strongly SLN-robust
if and only if it satisfies Equation \ref{eqn:constant}.
\end{proposition}

%


We now return to our original goal, which was to find a convex $\ell$ that is SLN-robust for $\LinearClass$ (and ideally more general function classes).
The above suggests that to do so, it is reasonable to consider as admissible those losses that satisfy Equation \ref{eqn:constant}.
Unfortunately, it is evident that if $\ell$ is convex, non-constant, and bounded below by zero, then it cannot possibly be admissible in this sense.
But we now show that removing the boundedness restriction allows for the existence of a convex admissible loss.

\section{The unhinged loss: a convex, classification-calibrated, strongly SLN-robust loss}
\label{sec:unhinged}

Consider the following simple, but non-standard convex loss:
\begin{align*}
\ellUnhinged_{1}(v) &= 1 - v \text{ and } \ellUnhinged_{-1}(v) = 1 + v.
\end{align*}
A peculiar property of the loss is that it is negatively unbounded, an issue we discuss in \S\ref{sec:taming-the-unhinge}.
Compared to the hinge loss, the loss does not clamp at zero, \ie it does not have a hinge.
Thus, we call this the \emph{unhinged loss}\footnote{This loss has been considered in \citet{Sriperumbudur:2009}, \citet{Reid:2011} in the context of maximum mean discrepancy; see \suppRef{sec:unhinged-mmd}.
The analysis of its SLN-robustness is to our knowledge novel.}\troll{\footnote{Alternately, as the loss involves removing the range restriction on the hinge loss, this is a \emph{de-ranged loss}.}}.
The loss has a number of attractive properties, the most immediate of which is its SLN-robustness.

\subsection{The unhinged loss is strongly SLN-robust}

Since $\ellUnhinged_1(v) + \ellUnhinged_{-1}(v) = 0$ we conclude from Proposition \ref{prop:eigen-MELT} that $\ellUnhinged$ is strongly SLN-robust, and thus that $( \ellUnhinged, \FCal )$ is SLN-robust for \emph{any} choice of $\FCal$.
Further, the following uniqueness property is not hard to show.

\begin{proposition}
\label{prop:everything-is-unhinged}
Pick any convex loss $\ell$.
Then,
$$ ( \exists C \in \Real ) \, \ell_1(v) + \ell_{-1}(v) = C \iff ( \exists A, B, D \in \Real ) \, \ell_1(v) = -A \cdot v + B, \ell_{-1}(v) = A \cdot v + D. $$
That is, up to scaling and translation, $\ellUnhinged$ is the only convex loss that is strongly SLN-robust.
\end{proposition}

Returning to the case of linear scorers,  
the above implies that 
$( \ellUnhinged, \LinearClass )$ is SLN-robust.
This does not contradict Proposition \ref{eqn:noise-defeats-all}, since $\ellUnhinged$ is not a convex potential as it is {negatively unbounded}.
Intuitively, this property allows the loss to compensate for the high penalty incurred by instances that are misclassified with high margin
by allowing for a high gain for instances that correctly classified with high margin.

%
\subsection{The unhinged loss is classification calibrated}

SLN-robustness is by itself insufficient for a learner to be useful.
For example, a loss that is uniformly zero is strongly SLN-robust, but is useless as it is not classification-calibrated.
Fortunately, the unhinged loss is classification-calibrated, as we now establish.
For reasons that shall be discussed subsequently, we consider minimisation of the risk over $\BoundedScorers = [ -B, +B ]^{\XCal}$,
the set of scorers with range bounded by $B \in [0, \infty)$.

\begin{proposition}
\label{prop:cc}
Fix $\ell = \ellUnhinged$.
Then, for any $\DMN$, $B \in [0, \infty)$,
$ \RestrictedOpt{\D, \BoundedScorers}{\ell} = \{ x \mapsto B \cdot \sign( 2\eta(x) - 1 ) \}. $
\end{proposition}

Thus, for every $B \in [0, \infty)$, the restricted Bayes-optimal scorer over $\BoundedScorers$ has the same sign as the Bayes-optimal classifier for 0-1 loss.
In the limiting case where $\FCal = \AllScorers$, the optimal scorer is attainable if we operate over the extended reals $\Real \cup \{ \pm \infty \}$, in which case we can conclude that $\ellUnhinged$ is classification-calibrated.


%
\subsection{Enforcing boundedness of the loss}
\label{sec:taming-the-unhinge}

While the classification-calibration of $\ellUnhinged$ is encouraging, Proposition \ref{prop:cc} implies that its (unrestricted) Bayes-risk is $-\infty$.
Thus, the {regret} of every non-optimal scorer $s$ is identically $+\infty$, which hampers analysis of consistency.
In orthodox decision theory, 
analogous theoretical issues arise when attempting to establish basic theorems with unbounded losses \citep[pg.\ 78]{Ferguson:1967}, \citep[pg.\ 172]{Schervish:1995}.

We can side-step this issue by restricting attention to bounded scorers, so that $\ellUnhinged$ is effectively bounded.
By Proposition \ref{prop:cc}, this does not affect the classification-calibration of the loss. 
In the context of linear scorers, boundedness of scorers can be achieved by regularisation:
instead of working with $\LinearClass$, one can instead use $\RegLinearClass = \{ x \mapsto \langle w, x \rangle \mid || w ||_2 \leq 1/\sqrt{\lambda} \}$, where $\lambda > 0$,
so that $\RegLinearClass \subseteq \FCal_{ R / \sqrt{\lambda} }$ for $R = \sup_{x \in \XCal} || x ||_2$. 
Observe that restricting to bounded scorers does \emph{not} affect the SLN-robustness of $\ellUnhinged$, because $( \ellUnhinged, \FCal )$ is SLN-robust for \emph{any} $\FCal$.
Thus, for example, $( \ellUnhinged, \RegLinearClass )$ is SLN-robust for any $\lambda > 0$.
As we shall see in \S\ref{sec:unhinged-regularised-svm}, working with $\RegLinearClass$ also lets us establish SLN-robustness of the hinge loss when $\lambda$ is large.


%

%
\subsection{Unhinged loss minimisation on corrupted distribution is consistent}

Using bounded scorers makes it possible to establish a surrogate regret bound for the unhinged loss.
This shows classification consistency of unhinged loss minimisation on the \emph{corrupted} distribution.

\begin{proposition}
\label{prop:surrogate-regret}
Fix $\ell = \ellUnhinged$.
Then, for any $\D, \rhoPlus \in [0, 1/2)$, $B \in [1, \infty)$, and scorer $s \in \BoundedScorers$,
$$ \regret^{\D}_{01}( s ) \leq \regret^{\D, \BoundedScorers}_{\ell}( s ) = \frac{1}{1 - 2\rhoPlus} \cdot \regret^{\ContamDShort, \BoundedScorers}_{\ell}( s ). $$
\end{proposition}

Standard rates of convergence via generalisation bounds are also trivial to derive; see \suppRef{sec:app-theory}.
We now turn to the question of how to minimise the unhinged loss when using a kernelised scorer.

\section{Learning with the unhinged loss and kernels}
\label{sec:relations}

We now show that the optimal solution for the unhinged loss when employing regularisation and kernelised scorers has a simple form. 
This sheds further light on SLN-robustness and regularisation.

\subsection{The centroid classifier optimises the unhinged loss}
\label{sec:unhinged-centroid}

Consider minimising the unhinged risk over some ball in a reproducing kernel Hilbert space $\HCal$ with kernel $k$, 
\ie consider the function class of kernelised scorers $\KernelScorers = \{ s \colon x \mapsto \langle w, \Phi( x ) \rangle_{\HCal} \mid || w ||_{\HCal} \leq 1/\sqrt{\lambda} \}$ for some $\lambda > 0$, where $\Phi \colon \XCal \to \HCal$ is some feature mapping.
Equivalently, given a distribution\footnote{Given a training sample $\SSf \sim \ContamDShort^n$, we can use plugin estimates as appropriate.} $\D$, we want
\begin{equation}
\label{eqn:empirical-risk}
\OptWeight{\unhinged} = \argminUnique{ w \in \HCal }{ \Expectation{( \X, \Y ) \sim \D}{1 - \Y \cdot \langle w, \Phi( \X ) \rangle} + \frac{\lambda}{2} \langle w, w \rangle_{\HCal} }.
\end{equation}
The first-order optimality condition implies that
\begin{equation}
\label{eqn:optimal-weight}
\OptWeight{\unhinged} = \frac{1}{\lambda} \cdot \Expectation{(\X, \Y) \sim \D}{ \Y \cdot \Phi( \X ) }.
\end{equation}
Thus, the optimal scorer for the unhinged loss is simply
\begin{equation}
\label{eqn:mmd}
\begin{aligned}
\OptScorer{\unhinged} &\colon x \mapsto \frac{1}{\lambda} \cdot \Expectation{(\X, \Y) \sim \D}{ \Y \cdot k( \X, x ) } = x \mapsto \frac{1}{\lambda} \cdot \left( \pi \cdot \Expectation{\X \sim P}{ k( \X, x ) } - (1 - \pi) \cdot \Expectation{\X \sim Q}{ k( \X, x ) } \right).
\end{aligned}
\end{equation}
That is, we score an instance based on the difference of the \emph{aggregate similarity to the positive instances, and the aggregate similarity to the negative instances}.
This is equivalent to a \emph{nearest centroid classifier} \citep[pg.\ 181]{Manning:2008} \citep{Tibshirani:2002}  \citep[Section 5.1]{Shawe-Taylor:2004}.
The quantity $\OptWeight{\unhinged}$ can be interpreted as the \emph{kernel mean map} of $\D$; see \suppRef{sec:app-relations} for more related work.

Equation \ref{eqn:mmd} gives a simple way to understand the SLN-robustness of $( \ellUnhinged, \KernelScorers )$:
it is easy to establish (see \suppRef{sec:app-mean-map}) that the optimal scorers on the clean and corrupted distributions only differ by a scaling, \ie
\begin{equation}
\label{eqn:mean-immunity}
( \forall x \in \XCal ) \, \Expectation{(\X, \Y) \sim \D}{ \Y \cdot k( \X, x ) } = \frac{1}{1 - 2\rhoPlus} \cdot \Expectation{(\X, \Contaminator{\Y}) \sim \ContamDShort}{ \Contaminator{\Y} \cdot k( \X, x ) }.
\end{equation}


%
\subsection{Practical considerations}

We note several points relating to practical usage of the unhinged loss with kernelised scorers. First, cross-validation is not required to select $\lambda$, 
since $\OptScorer{\unhinged}$ depends trivially on the regularisation constant: changing  $\lambda$ only changes the magnitude of scores, \emph{not their sign}.
Thus, regularisation simply controls the scale of the predicted scores, and for the purposes of classification, one can simply use $\lambda = 1$.

Second, we can easily extend the scorers to use a bias regularised with strength $0 < \lambda_b \neq \lambda$.
Tuning $\lambda_b$ is equivalent to computing $\OptScorer{\unhinged}$ as per Equation \ref{eqn:mmd}, and tuning a threshold on a holdout set.

Third, when $\HCal = \Real^d$ for $d$ small, we can store $\OptWeight{\unhinged}$ explicitly, and use this to make predictions.
For high (or infinite) dimensional $\HCal$, we can make predictions directly via Equation \ref{eqn:mmd}.
However, when learning with a training sample $\SSf \sim \D^n$, this would require storing the \emph{entire} sample for use at test time, which is undesirable.
To alleviate this, for a translation-invariant kernel one can use random Fourier features \citep{Rahimi:2007} to find an approximate embedding of $\HCal$ into some low-dimensional $\Real^d$, and then store $\OptWeight{\unhinged}$ as usual.
Alternately, one can \emph{post hoc} search for a sparse approximation to $\OptWeight{\unhinged}$, for example using kernel herding \citep{Chen:2012}.

We now show that under some assumptions, $\OptWeight{\unhinged}$ coincides with the solution of two established methods;
\suppRef{sec:app-relations} discusses some further relationships, \eg to  the maximum mean discrepancy.

\subsection{Equivalence to a highly regularised SVM and other convex potentials}
\label{sec:unhinged-regularised-svm}

There is an interesting equivalence between the unhinged solution and that of a \emph{highly regularised SVM}.

\begin{proposition}
\label{prop:highly-regularised}
Pick any $\D$ and $\Phi \colon \XCal \to \HCal$ such that $R = \sup_{x \in \XCal} || \Phi( x ) ||_{\HCal} < \infty$.
For any $\lambda > 0$, let
$$ \OptWeight{\hinge} = \argminUnique{ w \in \HCal }{ \Expectation{( \X, \Y ) \sim D}{ \max( 0, 1 - \Y \cdot \langle w, \Phi( x ) \rangle_{\HCal} ) } + \frac{\lambda}{2} \langle w, w \rangle_{\HCal} } $$
be the soft-margin SVM solution.
Then, if $\lambda \geq R^2$, $w^*_{\mathrm{hinge}, \lambda} = \OptWeight{\unhinged}$.
\end{proposition}

Since we know that $( \ellUnhinged, \KernelScorers )$ is SLN-robust, it follows immediately that
for $\ellHinge \colon (y, v) \mapsto \max(0, 1 - yv)$, 
$( \ellHinge, \KernelScorers )$ is similarly SLN-robust \emph{provided $\lambda$ is sufficiently large}.
That is, strong $\ell_2$ regularisation (and a bounded feature map) endows the hinge loss with SLN-robustness\footnote{By contrast, \citet[Section 6]{Long:2010} establish that $\ell_1$ regularisation does not endow SLN-robustness.}.

Proposition \ref{prop:highly-regularised} can be generalised to show that with sufficiently strong regularisation, the limiting solution of \emph{any} twice differentiable convex potential will be 
the unhinged solution, \ie, the centroid classifier.
Intuitively, with strong regularisation, one only considers the behaviour of a loss near zero;
but since a convex potential $\phi$ has $\phi'( 0 ) < 0$, it will be well-approximated by the unhinged loss near zero (which is simply the linear approximation to $\phi$).
This shows that \emph{strong $\ell_2$ regularisation endows most learners with SLN-robustness}.

\begin{proposition}
\label{prop:really-everything-is-unhinged}
Pick any $\D$, bounded feature mapping $\Phi \colon \XCal \to \HCal$, 
and twice differentiable convex potential $\phi$.
Let $\OptWeight{\phi}$ be the minimiser of the regularised $\phi$ risk.
Then,
$$ ( \forall \epsilon > 0 ) \, ( \exists \lambda_0 > 0 ) \, ( \forall \lambda > \lambda_0 ) \, \left|\left| \frac{\OptWeight{\phi}}{|| \OptWeight{\phi} ||_{\HCal}} - \frac{\OptWeight{\unhinged}}{|| \OptWeight{\unhinged} ||_{\HCal}} \right|\right|_{\HCal}^2 \leq \epsilon. $$
\end{proposition}

\subsection{Equivalence to Fisher Linear Discriminant with whitened data}

Recall that for binary classification on $\DMN$, the Fisher Linear Discriminant (FLD) finds a weight vector proportional to the minimiser of square loss $\ellSquare \colon (y, v) \mapsto (1 - y v)^2$ \citep[Section 4.1.5]{Bishop:2006},
\begin{equation}
\label{eqn:flda}
\OptWeight{\square} = ( \mathbb{E}_{\X \sim M}[ \X \X^T ] + \lambda I )^{-1} \cdot \mathbb{E}_{(\X, \Y) \sim D}[ \Y \cdot \X ].
\end{equation}
By Equation \ref{eqn:mean-immunity}, and the fact that the corrupted marginal $\MCont = M$, we see that $\OptWeight{\square}$ is only changed by a scaling factor under label noise.
This provides an alternate proof of the fact that $( \ellSquare, \LinearClass )$ is SLN-robust\footnote{Square loss escapes the result of \citet{Long:2010} since it is not monotone decreasing.} \citep[Theorem 2]{Manwani:2013}. 

Clearly, the unhinged loss solution $ \OptWeight{\unhinged} $ is equivalent to the FLD and square loss solution $\OptWeight{\square}$   when the input data is whitened \ie $\Expectation{\X \sim M}{ \X \X^T } = I$. 
With a well-specified $\FCal$, \eg with a universal kernel, both the unhinged and square loss asymptotically recover the optimal classifier, but the unhinged loss does not require a matrix inversion. 
With a misspecified $\FCal$, one cannot in general argue for the superiority of the unhinged loss over square loss, or vice-versa, as there is no universally good surrogate to the 0-1 loss \citep[Appendix A]{Reid:2010}; 
\suppRef{sec:app-breaking-the-square}, \suppRef{sec:app-breaking-the-unhinged} illustrate examples where both losses may underperform.



\section{SLN-robustness of unhinged loss: empirical illustration}

We now illustrate that the SLN-robustness of the unhinged loss is empirically manifest. 
We reiterate that with high regularisation, the unhinged solution is equivalent to an SVM (and in the limit to any classification-calibrated loss) solution.
Thus, the experiments do \emph{not} aim to assert that the unhinged loss is ``better'' than other losses, but rather,
to demonstrate that its SLN-robustness is not \emph{purely} theoretical.

%
\label{sec:long-expts}
We first show that the unhinged risk minimiser performs well on the example of \citet{Long:2010}.
Figure \ref{fig:long} shows the distribution $\D$, where $\XCal = \{ ( 1, 0 ), ( \gamma, 5 \gamma ), ( \gamma, -\gamma ) \} \subset \Real^{2}$, with marginal distribution $M = \{ \frac{1}{4}, \frac{1}{4}, \frac{1}{2} \}$ and all three instances are deterministically positive.
We pick $\gamma = {1}/{2}$.
From Figure \ref{fig:long}, we see the unhinged minimiser perfectly classifies all three points, regardless of the level of label noise.
The hinge risk minimiser is perfect when there is no label noise, but 
with even a small amount of label noise, 
achieves an error rate of 50\%.

\begin{minipage}{\textwidth}
\begin{minipage}[c]{0.49\textwidth}

	\centering
	
\begin{tikzpicture}[ultra thick, scale=0.8]	

\begin{axis}[xmin=-0.3,xmax=1.3,xtick={0,0.5,1},axis lines = middle,legend cell align=left,legend style={at={(axis cs:1.5,1)}}]
\addplot[color=red,ultra thick] coordinates {
(-1/5, 1)
(1/5,-1)
};
\addplot[color=blue,dashed,ultra thick] coordinates {
(0.0025, 1)
(-0.0025,-1)
};
\addplot[color=brown,dotted,ultra thick] coordinates {
(0.6581, 1)
(-0.6581,-1)
};
\addplot[black,mark=*] coordinates {(1, 0)};
\addplot[black,mark=*] coordinates {(1/12, 5/12)};
\addplot[black,mark=*,mark size=4] coordinates {(1/12, -1/12)};
\legend{Unhinged, Hinge 0\% noise, Hinge 1\% noise}
\end{axis}

\end{tikzpicture}
	\captionof{figure}{\citet{Long:2010} dataset.}
	\label{fig:long}
\end{minipage}	
\quad
\begin{minipage}[c]{0.49\textwidth}
{\scriptsize
\begin{tabular}{@{}llll@{}}
\toprule
\toprule
& \textbf{Hinge} & \textbf{$t$-logistic} & \textbf{Unhinged} \\ 
\midrule
$\rhoPlus = 0$ & \cellcolor{gray!25}{0.00 $\pm$ 0.00} & \cellcolor{gray!25}{0.00 $\pm$ 0.00} & \cellcolor{gray!25}{0.00 $\pm$ 0.00} \\ 
$\rhoPlus = 0.1$ & 0.15 $\pm$ 0.27 & \cellcolor{gray!25}{0.00 $\pm$ 0.00} & \cellcolor{gray!25}{0.00 $\pm$ 0.00} \\ 
$\rhoPlus = 0.2$ & 0.21 $\pm$ 0.30 & \cellcolor{gray!25}{0.00 $\pm$ 0.00} & \cellcolor{gray!25}{0.00 $\pm$ 0.00} \\ 
$\rhoPlus = 0.3$ & 0.38 $\pm$ 0.37 & 0.22 $\pm$ 0.08 & \cellcolor{gray!25}{0.00 $\pm$ 0.00} \\ 
$\rhoPlus = 0.4$ & 0.42 $\pm$ 0.36 & 0.22 $\pm$ 0.08 & \cellcolor{gray!25}{0.00 $\pm$ 0.00} \\ 
$\rhoPlus = 0.49$ & 0.47 $\pm$ 0.38 & 0.39 $\pm$ 0.23 & \cellcolor{gray!25}{0.34 $\pm$ 0.48} \\ 
\midrule
\end{tabular}
}
\captionof{table}{Mean and standard deviation of the 0-1 error over 125 trials on \citet{Long:2010}. Grayed cells denote the best performer at that noise rate.}
\label{tbl:long-matlab}
\end{minipage}
\end{minipage}

We next consider minimisers of the empirical risk from a random training sample:
we construct a training set of $800$ instances, injected with varying levels of label noise, and evaluate classification performance on a test set of $1000$ instances.
We compare the hinge, $t$-logistic (for $t = 2$) \citep{Ding:2010} and unhinged minimisers.
For each loss, we use a linear scorer \emph{without} a bias term, and set the regularisation strength $\lambda = 10^{-16}$.
From Table \ref{tbl:long-matlab}, it is apparent that even at 40\% label noise, the unhinged classifier is able to find a perfect solution.
By contrast, both other losses suffer at even moderate noise rates.

%

We next report results on some UCI datasets, 
where we additionally tune a threshold so as to ensure the best training set 0-1 accuracy.
Table \ref{tbl:uci-matlab} summarises results on a sample of four datasets.
(\suppRef{sec:app-expts} contains results with more datasets, performance metrics, and losses.)
While the unhinged loss is sometimes outperformed at low noise, it tends to be much more robust at high levels of noise: 
even at noise close to 50\%, it is often able to learn a classifier with some discriminative power.

\begin{table}[htb]
	\centering
	\renewcommand{\arraystretch}{1.25}
	
	\subfloat[\texttt{iris}.]{\scriptsize
\begin{tabular}{@{}llllll@{}}
\toprule
\toprule
& \textbf{Hinge} & \textbf{$t$-Logistic} & \textbf{Unhinged} \\ 
\midrule
$\rhoPlus = 0$ & \cellcolor{gray!25}{0.00 $\pm$ 0.00} & \cellcolor{gray!25}{0.00 $\pm$ 0.00} & \cellcolor{gray!25}{0.00 $\pm$ 0.00} \\ 
$\rhoPlus = 0.1$ & 0.01 $\pm$ 0.03 & 0.01 $\pm$ 0.03 & \cellcolor{gray!25}{0.00 $\pm$ 0.00} \\ 
$\rhoPlus = 0.2$ & 0.06 $\pm$ 0.12 & 0.04 $\pm$ 0.05 & \cellcolor{gray!25}{0.00 $\pm$ 0.01} \\ 
$\rhoPlus = 0.3$ & 0.17 $\pm$ 0.20 & 0.09 $\pm$ 0.11 & \cellcolor{gray!25}{0.02 $\pm$ 0.07} \\ 
$\rhoPlus = 0.4$ & 0.35 $\pm$ 0.24 & 0.24 $\pm$ 0.16 & \cellcolor{gray!25}{0.13 $\pm$ 0.22} \\ 
$\rhoPlus = 0.49$ & 0.60 $\pm$ 0.20 & 0.49 $\pm$ 0.20 & \cellcolor{gray!25}{0.45 $\pm$ 0.33} \\ 
\midrule
\end{tabular}
	}
	\quad
	\subfloat[\texttt{housing}.]{\scriptsize
\begin{tabular}{@{}llllll@{}}
\toprule
\toprule
& \textbf{Hinge} & \textbf{$t$-Logistic} & \textbf{Unhinged} \\ 
\midrule
$\rhoPlus = 0$ & \cellcolor{gray!25}{0.05 $\pm$ 0.00} & \cellcolor{gray!25}{0.05 $\pm$ 0.00} & \cellcolor{gray!25}{0.05 $\pm$ 0.00} \\ 
$\rhoPlus = 0.1$ & 0.06 $\pm$ 0.01 & 0.07 $\pm$ 0.02 & \cellcolor{gray!25}{0.05 $\pm$ 0.00} \\ 
$\rhoPlus = 0.2$ & 0.06 $\pm$ 0.01 & 0.08 $\pm$ 0.03 & \cellcolor{gray!25}{0.05 $\pm$ 0.00} \\ 
$\rhoPlus = 0.3$ & 0.08 $\pm$ 0.04 & 0.11 $\pm$ 0.05 & \cellcolor{gray!25}{0.05 $\pm$ 0.01} \\ 
$\rhoPlus = 0.4$ & 0.14 $\pm$ 0.10 & 0.24 $\pm$ 0.13 & \cellcolor{gray!25}{0.09 $\pm$ 0.10} \\ 
$\rhoPlus = 0.49$ & \cellcolor{gray!25}{0.45 $\pm$ 0.26} & 0.49 $\pm$ 0.16 & 0.46 $\pm$ 0.30 \\ 
\midrule
\end{tabular}
	}
\vspace{-0.1in}
	\subfloat[\texttt{usps0v7}.]{\scriptsize
\begin{tabular}{@{}llllll@{}}
\toprule
\toprule
& \textbf{Hinge} & \textbf{$t$-Logistic} & \textbf{Unhinged} \\ 
\midrule
$\rhoPlus = 0$ & \cellcolor{gray!25}{0.00 $\pm$ 0.00} & \cellcolor{gray!25}{0.00 $\pm$ 0.00} & \cellcolor{gray!25}{0.00 $\pm$ 0.00} \\ 
$\rhoPlus = 0.1$ & 0.10 $\pm$ 0.08 & 0.11 $\pm$ 0.02 & \cellcolor{gray!25}{0.00 $\pm$ 0.00} \\ 
$\rhoPlus = 0.2$ & 0.19 $\pm$ 0.11 & 0.15 $\pm$ 0.02 & \cellcolor{gray!25}{0.00 $\pm$ 0.00} \\ 
$\rhoPlus = 0.3$ & 0.31 $\pm$ 0.13 & 0.22 $\pm$ 0.03 & \cellcolor{gray!25}{0.01 $\pm$ 0.00} \\ 
$\rhoPlus = 0.4$ & 0.39 $\pm$ 0.13 & 0.33 $\pm$ 0.04 & \cellcolor{gray!25}{0.02 $\pm$ 0.02} \\ 
$\rhoPlus = 0.49$ & 0.50 $\pm$ 0.16 & 0.48 $\pm$ 0.04 & \cellcolor{gray!25}{0.34 $\pm$ 0.21} \\ 
\midrule
\end{tabular}
	}	
	\quad
	\subfloat[\texttt{splice}.]{\scriptsize
\begin{tabular}{@{}llllll@{}}
\toprule
\toprule
& \textbf{Hinge} & \textbf{$t$-Logistic} & \textbf{Unhinged} \\ 
\midrule
$\rhoPlus = 0$ & 0.05 $\pm$ 0.00 & \cellcolor{gray!25}{0.04 $\pm$ 0.00} & 0.19 $\pm$ 0.00 \\ 
$\rhoPlus = 0.1$ & \cellcolor{gray!25}{0.15 $\pm$ 0.03} & 0.24 $\pm$ 0.00 & 0.19 $\pm$ 0.01 \\ 
$\rhoPlus = 0.2$ & 0.21 $\pm$ 0.03 & 0.24 $\pm$ 0.00 & \cellcolor{gray!25}{0.19 $\pm$ 0.01} \\ 
$\rhoPlus = 0.3$ & 0.25 $\pm$ 0.03 & 0.24 $\pm$ 0.00 & \cellcolor{gray!25}{0.19 $\pm$ 0.03} \\ 
$\rhoPlus = 0.4$ & 0.31 $\pm$ 0.05 & 0.24 $\pm$ 0.00 & \cellcolor{gray!25}{0.22 $\pm$ 0.05} \\ 
$\rhoPlus = 0.49$ & 0.48 $\pm$ 0.09 & \cellcolor{gray!25}{0.40 $\pm$ 0.24} & 0.45 $\pm$ 0.08 \\ 
\midrule
\end{tabular}
	}
	
%
	
	\caption{Mean and standard deviation of the 0-1 error over 125 trials on UCI datasets.}
	\label{tbl:uci-matlab}
	
\end{table}

\section{Conclusion and future work}

We have proposed a convex, classification-calibrated loss, proved that is robust to symmetric label noise (SLN-robust),
shown it is the unique loss that satisfies a notion of strong SLN-robustness,
established that it is optimised by the nearest centroid classifier,
and also shown how the nature of the optimal solution implies that most convex potentials, such as the SVM, are also SLN-robust when highly regularised.
Future work includes studying losses robust to asymmetric noise, and outliers in instance space.

\clearpage

\subsubsection*{Acknowledgments} 
NICTA is funded by the Australian Government through the Department of Communications and the Australian Research Council through the ICT Centre of Excellence Program.
The authors thank Cheng Soon Ong for valuable comments on a draft of this paper.

\newpage

\appendix

\onecolumn

{\LARGE
\begin{center}
\textbf{Proofs for ``Learning with Symmetric Label Noise: The Importance of Being Unhinged''}
\end{center}
}

\section{Proofs of results in main body}
\label{sec:proofs}

We now present proofs of all results in the main body.

\begin{proof}[Proof of Proposition \ref{eqn:noise-defeats-all}]
This result is stated implicitly in \citet[Theorem 2]{Long:2010}; 
the aim of this proof is simply to make the result explicit.

Let $\XCal = \{ ( 1, 0 ), ( \gamma, 5 \gamma ), ( \gamma, -\gamma ), ( \gamma, -\gamma ) \} \subset \Real^{2}$, for some $\gamma < 1/6$.
Let the marginal distribution over $\XCal$ be uniform.
Let $\eta \colon x \mapsto 1$, \ie let every example be deterministically positive.

Now suppose we observe some $\CCNDSymm$, for $\rhoPlus \in [0, 1/2)$.
We minimise the $\ell$-risk some convex potential $\ell \colon (y, v) \mapsto \phi(y, v)$ using a linear function class\footnote{The result actually requires that one not include a bias term; with a bias term, it can be checked that the example as-stated has a trivial solution.} $\LinearClass$.
Then, \citet[Theorem 2]{Long:2010} establishes that
$$ ( \forall s \in \RestrictedOpt{\ContamDShort, \LinearClass}{\ell} ) \, \Risk^{\D}_{01}( s ) = \frac{1}{2}. $$
On the other hand, since $\D$ is linearly separable and a convex potential $\ell$ is classification-calibrated, we must have $\Risk^{\D}_{01}( \RestrictedOpt{\D, \LinearClass}{\ell} ) = 0$.
Consequently, for any convex potential $\ell$, $( \ell, \LinearClass )$ is not SLN-robust.
\end{proof}

\begin{proof}[Proof of Proposition \ref{eqn:kernel-defeats-noise}]
Let $\etaContam$ be the class-probability function of $\ContamDShort$.
By \cite[Lemma 7]{Natarajan:2013},
$$ ( \forall x \in \XCal ) \, \sign( 2\etaContam(x) - 1 ) = \sign( 2\eta(x) - 1 ), $$
so that the optimal classifiers on the clean and corrupted distributions coincide.
Therefore, intuitively, if the Bayes-optimal solution for loss recovers $\sign( 2\etaContam(x) - 1 )$, it will also recover $\sign( 2\eta(x) - 1 )$.
Formally, since $\ell$ is classification-calibrated, for any $\D \in \AllDists$, and $s \in \RestrictedOpt{\D}{\ell}$
$$ ( \forall x \in \XCal ) \, \sign( s( x ) ) = \sign( 2\eta(x) - 1 ), $$
and similarly, for any $\ContamDShort \in \RobustCCNSymm( \D )$, and $\bar{s} \in \RestrictedOpt{\ContamDShort}{\ell}$
$$ ( \forall x \in \XCal ) \, \sign( \bar{s}( x ) ) = \sign( 2\etaContam(x) - 1 ). $$
Thus, for any $\D, \ContamDShort$, since the 0-1 risk of a scorer depends only on its sign,
\begin{align*}
\Risk^{\D}_{01}( s ) &= \Risk^{\D}_{01}( s ) \\
&= \Risk^{\D}_{01}( \sign( 2\eta - 1 ) ) \\
&= \Risk^{\D}_{01}( \sign( 2\etaContam - 1 ) ) \\
&= \Risk^{\D}_{01}( \bar{s} ).
\end{align*}
Consequently,
\ie $( \ell, \AllScorers )$ is SLN-robust.
\end{proof}

\begin{proof}[Proof of Proposition \ref{prop:eigen-MELT}]
$(\impliedby)$. If $\ell$ satisfies Equation \ref{eqn:constant}, then its noise corrected counterpart is
$$ ( \forall y \in \PMOne ) ( \forall v \in \Real ) \, \ellContam(y, v) = \frac{1}{1 - 2\rhoPlus} \cdot \ell(y, v) - C \cdot \frac{\rho}{1 - 2\rhoPlus}, $$
that is, it is a scaled and translated version of $\ell$.
Consequently, for any $\rhoPlus$, the corresponding risk will be a scaled and translated version of the $\ell$-risk.
It is immediate that the two losses will be order equivalent for any $\rhoPlus$.

$(\implies) $. Recall that $\SSf$ denotes the distribution of scores.
For any stochastic scorer $f$, let
$$ S_f \colon a \mapsto \Pr( \SSf = a ) $$
be the corresponding marginal distribution of scores.
Similarly, let
$$ M_a \colon x \mapsto \Pr( \X = x \mid \SSf = a ) $$
be the conditional distribution of instances given a predicted score $a \in \Real$.
Finally, for any $a \in \Real$, let $\D_a = ( M_a, \eta )$ be an induced distribution over $\XCal \times \PMOne$.

With the above, we can rewrite the stochastic risk as
\begin{align*}
\Risk^{\D}_{\ell}( f ) &= \Expectation{ \SSf \sim S_f }{ \Expectation{( \X, \Y ) \sim \D_\SSf}{\ell( \Y, \SSf )} } \\
&= \Expectation{ \SSf \sim S_f }{ \Risk^{\D_{\SSf}}_{\ell}( \SSf ) }.
\end{align*}
That is, we average, over all achievable scores according to $f$, the risks of that constant prediction on an appropriately reweighed version of the original distribution $\D$.
Then, for some fixed $\rhoPlus \in [0, 1/2)$, the fact that $\ell$ and $\ellContam$ are order equivalent can be written
$$ ( \forall \D ) \, ( \forall f, g \in \Delta_{\Real}^{\XCal} ) \, \Expectation{ \SSf \sim S_f }{ \Risk^{\D_{\SSf}}_{\ell}( \SSf ) } \leq \Expectation{ \SSf \sim S_g }{ \Risk^{\D_{\SSf}}_{\ell}( \SSf ) } \iff \Expectation{ \SSf \sim S_f }{ \Risk^{\D_{\SSf}}_{\ellContam}( \SSf ) } \leq \Expectation{ \SSf \sim S_g }{ \Risk^{\D_{\SSf}}_{\ellContam}( \SSf ) }. $$

Now define the utility functions
$$ U^{\D} \colon a \mapsto -\EllRisk{\D_a}{ a } $$
and
$$ V^{\D} \colon a \mapsto -\Risk^{\D_a}_{\ellContam}( a ). $$
Then, order equivalence can be trivially re-expressed as
$$ ( \forall \D ) \, ( \forall f, g \in \Delta_{\Real}^{\XCal} ) \Expectation{ \SSf \sim S_f }{ U^{\D}( \SSf ) } \geq \Expectation{ \SSf \sim S_g }{ U^{\D}( \SSf ) } \iff \Expectation{ \SSf \sim S_f } { V^{\D}( \SSf ) } \geq \Expectation{ \SSf \sim S_g }{ V^{\D}( \SSf ) }. $$

That is, for any fixed distribution $\D$, the utility functions $U^{\D}, V^{\D}$ specify the same ordering over distributions in $\Delta_{\Real}$.
Therefore, by \citet[Section 7.9, Theorem 2]{deGroot:1970}, for any fixed $\D$, they must be affinely related:
$$ ( \forall \D ) \, ( \exists \alpha, \beta \in \Real ) \, ( \forall a \in \Real ) \, U^{\D}(a) = \alpha \cdot V^{\D}(a) + \beta. $$

Converting this back to losses, and using the definition of strong SLN-robustness,
$$ ( \forall \D ) \, ( \forall \rhoPlus \in [0, 1/2) ) \, ( \exists \alpha, \beta \in \Real ) \, ( \forall a \in \Real ) \, \EllRisk{\D_a}{ a } = \alpha \cdot \Risk^{\D_a}_{\ellContam}( a ) + \beta $$
or
$$ ( \forall \D ) \, ( \forall \rhoPlus \in [0, 1/2) ) \, ( \exists \alpha, \beta \in \Real ) \, ( \forall a \in \Real ) \, \Expectation{( \X, \Y ) \sim \D_a}{ \ell( \Y, a ) - ( \alpha \cdot \ellContam( \Y, a ) + \beta ) } = 0. $$
For this to hold for all possible $\D$, it must be true that
$$ (\forall \rhoPlus \in [0, 1/2) ) \, ( \exists \alpha, \beta \in \Real ) \, ( \forall y, v ) \, \ell(y, v) = \alpha \cdot \ellContam(y, v) + \beta. $$
By Lemma \ref{lemm:sum-constant}, the result follows.
\end{proof}

\begin{proof}[Proof of Proposition \ref{prop:everything-is-unhinged}]
($\impliedby$). Clearly for an $\ell$ satisfying the given condition, $\ell_1(v) + \ell_{-1}(v) = B + C$, a constant.

($\implies$). By assumption, $\ell_1$ is convex. By the given condition, equivalently, $( \exists C \in \Real ) \, C - \ell_{1}$ is convex. But this is in turn equivalent to $-\ell_1$ also being convex. The only possibility for both $\ell_1$ and $-\ell_1$ being convex is that $\ell_1$ is affine, hence showing the desired implication.
\end{proof}

\begin{proof}[Proof of Proposition \ref{prop:cc}]
Fix $\ell = \ellUnhinged$.
It is easy to check that
\begin{equation}
\label{eqn:unhinged-cond-risk}
( \forall \eta \in [0, 1] ) \, ( \forall v \in \Real ) \, L_\ell (\eta, v) = (1 - 2\eta) \cdot v,
\end{equation}
and so
$$ ( \forall \eta \in [0, 1] ) \, \argminUnique{ v \in [ -B, +B ] }{ L_{\ell}( \eta, v ) } = \begin{cases} +B & \text{ if } \eta > \frac{1}{2} \\ -B & \text{ else. } \end{cases} $$
\end{proof}

\begin{proof}[Proof of Proposition \ref{prop:surrogate-regret}]
Fix $\ell = \ellUnhinged$.
Since by Equation \ref{eqn:unhinged-cond-risk} $L_{\ellUnhinged}(\eta, v) = (1 - 2 \eta) \cdot v$, we have that
$$ \mathbb{L}^{\D}_{\ell}( s ) = - \Expectation{ \X \sim M }{ (2\eta(\X) - 1) \cdot s(\X) }, $$
and since the restricted Bayes-optimal scorer is $x \mapsto B \cdot \sign( 2\eta(x) - 1 )$,
$$ \mathbb{L}^{\D, \FCal_{B}, *}_{\ell} = - B \cdot \Expectation{ \X \sim M }{ | 2\eta(\X) - 1 | }. $$
Thus,
$$ \regret^{\D, \FCal_{B}}_{\ell}( s ) = \Expectation{ \X \sim M }{ | 2\eta(\X) - 1 | \cdot ( B - s(\X) \cdot \sign( 2\eta(\X) - 1 ) ) } $$

Now, since the scorer $x \mapsto \sign( 2\eta(x) - 1 ) \in \BayesOpt{\D}{01} \cap \FCal_{B}$, we have that $\regret^{\D, \FCal_{B}}_{01}( s ) = \regret^{\D}_{01}( s )$.
Further, we have that
$$ \regret^{\D}_{01}( s ) = \Expectation{ \X \sim M }{ | 2\eta(\X) - 1  | \cdot \indicator{ s(\X) \cdot \sign( 2 \eta(X) - 1 ) < 0 } }. $$
But if $B \geq 1$,
$$ \indicator{ v  < 0 } \leq B - v.  $$
Thus,
$$ \regret^{\D, \FCal_{B}}_{01}( s ) \leq \regret^{\D, \FCal_{B}}_{\ell}( s ). $$

Finally, by Equation \ref{eqn:melted-loss}, for $\ell = \ellUnhinged$,
$$ ( \forall y \in \PMOne ) \, ( \forall v \in \Real ) \, \ellContam(y, v) = \frac{1}{1 - 2\rhoPlus} \cdot \ell(y, v), $$
\ie the unhinged loss is its own noise-corrected loss, with a scaling factor of $\frac{1}{1 - 2\rhoPlus}$.
Thus, since the $\ell$-regret on $\D$ and $\ellContam$-regret on $\ContamDShort$ coincide,
$$ \regret^{\D, \FCal_{B}}_{\ell}( s ) = \regret^{\ContamDShort, \FCal_{B}}_{\ellContam}( s ) = \frac{1}{1 - 2\rhoPlus} \cdot \regret^{\ContamDShort, \FCal_{B}}_{\ell}( s ). $$
\end{proof}

\begin{proof}[Proof of Proposition \ref{prop:highly-regularised}]
On a distribution $\D$, a soft-margin SVM solves
$$ \min_{ w \in \HCal } \Expectation{( \X, \Y ) \sim D}{ \max( 0, 1 - \Y \cdot \langle w, \Phi( x ) \rangle_{\HCal} ) } + \frac{\lambda}{2} \langle w, w \rangle_{\HCal}^2. $$
Let $\OptWeight{\hinge}$ denote the optimal solution to this objective. Now, by \citet[Theorem 1]{Shalev-Shwartz:2007},
$$ || \OptWeight{\hinge} ||_{\HCal} \leq \frac{1}{\sqrt{\lambda}}. $$
Now suppose that $R = \sup_{x \in \XCal} || \Phi( x ) ||_{\HCal} < \infty$.
Then, by the Cauchy-Schwartz inequality,
$$ ( \forall x \in \XCal) \, | \langle \OptWeight{\hinge}, \Phi( x ) \rangle_{\HCal} | \leq || \OptWeight{\hinge} ||_{\HCal} \cdot || \Phi( x ) ||_{\HCal} \leq \frac{R}{\sqrt{\lambda}}. $$
It follows that if $\lambda \geq R^2$, then
$$ ( \forall x \in \XCal) \, | \langle \OptWeight{\hinge}, \Phi( x ) \rangle_{\HCal} | \leq 1. $$
But this means that we never activate the flat portion of the hinge loss. Thus, for $\lambda \geq R^2$, the SVM objective is equivalent to
$$ \min_{ w \in \HCal } \Expectation{( \X, \Y ) \sim D}{ 1 - \Y \cdot \langle w, \Phi( x ) \rangle_{\HCal} } + \frac{\lambda}{2} \langle w, w \rangle_{\HCal}^2. $$
which means the optimal solution will coincide with that of the regularised unhinged loss.
Therefore, we can view unhinged loss minimisation as corresponding to learning a highly regularised SVM\footnote{This also holds if we add a regularised bias term.
With an unregularised bias term, \citet{Bedo:2006} showed that the limiting solution of a soft-margin SVM is distribution dependent.}.
\end{proof}

\begin{proof}[Proof of Proposition \ref{prop:really-everything-is-unhinged}]
Fix some distribution $\D$.
Let
$$ \mu = \Expectation{( \X, \Y ) \sim \D}{ \Y \cdot \Phi( \X ) } $$
be the optimal unhinged solution with regularisation strength $\lambda = 1$.
Observe that $|| \mu ||_{\HCal} \leq R = \sup_{x \in \XCal} || \Phi( x ) ||_{\HCal} < \infty$.
For some $r > 0$, let
$$ w^*_{\phi} = \argminUnique{ || w ||_{\HCal} \leq r }{ \Risk^{\D}_{\phi}( w ) } $$
be the optimal $\phi$ solution with norm bounded by $r$.
Similarly, let
$$ w^*_{\unhinged} = || w^*_{\phi} || \cdot \frac{\mu}{|| \mu ||_{\HCal}} $$
be the optimal unhinged solution with the same norm as the optimal $\phi$ solution.
We will show that these two vectors have similar unhinged risks, and use this to show that the corresponding unit vectors must be close.


By definition, a convex potential has $\phi'( 0 ) < 0$.
Without loss of generality, we can scale the potential so that $\phi'(0) = -1$.
Then, since $\phi$ is convex, it is lower bounded by the linear approximation at zero:
$$ (\forall v \in \Real) \, \phi( v ) - \phi( 0 ) \geq - v. $$
Observe that the RHS is the unhinged loss.
Thus, 
the unhinged risk can be bounded by its $\phi$ counterpart.
In particular, at the optimal $\phi$ solution,
$$ \Risk_{\unhinged}( w^*_{\phi} ) \leq \Risk_{\phi}( w^*_{\phi}) - \phi( 0 ). $$
Therefore, the difference between the unhinged and $\phi$ optimal solutions is
\begin{align}
\label{eqn:u-phi-diff}
\nonumber \Risk_{\unhinged}( w^*_{\phi} ) - \Risk_{\unhinged}( w^*_{\unhinged} ) &\leq \Risk_{\phi}( w^*_{\phi} ) - \Risk_{\unhinged}( w^*_{\unhinged} ) - \phi( 0 ) \\
\nonumber &\leq \Risk_{\phi}( w^*_{\unhinged} ) - \Risk_{\unhinged}( w^*_{\unhinged} ) - \phi( 0 ) \\
\nonumber &= \Expectation{( \X, \Y ) \sim \D}{ \phi( \Y \langle w^*_{\unhinged}, \Phi( \X ) \rangle_{\HCal} ) + \Y \langle w^*_{\unhinged}, \Phi( \X ) \rangle_{\HCal} } - \phi( 0 ) \\
&= \Expectation{( \X, \Y ) \sim \D}{ \tilde{\phi}( \Y \langle w^*_{\unhinged} , \Phi( \X ) \rangle_{\HCal} ) },
\end{align}
where $\tilde{\phi} \colon v \mapsto \phi( v ) - \phi( 0) + v$.
(The second line follows by definition of optimality of $w^*_{\phi}$ amongst all vectors with norm bounded by $r$.)
We have already established that $\tilde{\phi} \geq 0$.
Now, by Taylor's remainder theorem,
\begin{equation}
\label{eqn:taylor}
( \forall v \in ( -1, 1 ) ) \, \tilde{\phi}( v ) \leq \frac{a}{2} v^2,
\end{equation}
where $a = \max_{v \in [ -1, 1 ]} \phi''( v )$. But by Cauchy-Schwartz, we can restrict attention in Equation \ref{eqn:u-phi-diff} to the behaviour of $\tilde{\phi}$ in the interval
$$I = [ - || w^*_{\unhinged}  ||_{\HCal} \cdot R, || w^*_{\unhinged}  ||_{\HCal} \cdot R ],$$
where $R = \sup_{x \in \XCal} || \Phi( x ) ||_{\HCal} < \infty$.
Therefore, if $r \leq \frac{1}{R}$, 
Equation \ref{eqn:taylor} and a further application of Cauchy-Schwartz yield
\begin{align*}
\Risk_{\unhinged}( w^*_{\phi} ) - \Risk_{\unhinged}( w^*_{\unhinged} ) &\leq \frac{a}{2} \cdot \Expectation{( \X, \Y ) \sim \D}{ \langle w^*_{\unhinged}, \Phi( \X ) \rangle_{\HCal}^2 ) } \\
&\leq \frac{a}{2} \cdot \Expectation{\X \sim M}{ || w^*_{\unhinged} ||_{\HCal}^2 \cdot || \Phi( \X ) ||_{\HCal}^2 } \\
&\leq \frac{a R^2}{2} \cdot || w^*_{\unhinged} ||_{\HCal}^2.
\end{align*}

Now, the unhinged risk is
$$ \Risk^{\D}_{\unhinged}( w ) = -\langle w, \mu \rangle_{\HCal}. $$
Thus,
$$ -\langle w^*_{\phi}, \mu \rangle_{\HCal} + \langle w^*_{\unhinged}, \mu \rangle_{\HCal} \leq \frac{a R^2}{2} \cdot || w^*_{\unhinged} ||_{\HCal}^2. $$
Rearranging, and by definition of $w^*_{\unhinged}$,
\begin{align*}
\langle w^*_{\phi}, \mu \rangle_{\HCal} &\geq \langle w^*_{\unhinged}, \mu \rangle_{\HCal} - \frac{a R^2}{2} \cdot || w^*_{\unhinged} ||_{\HCal}^2 \\
&= ||w^*_{\phi} ||_{\HCal} \cdot || \mu ||_{\HCal} - \frac{a R^2}{2} \cdot || w^*_{\phi} ||_{\HCal}^2 \\
&= ||w^*_{\phi} ||_{\HCal} \cdot || \mu ||_{\HCal} \cdot \left( 1 - \frac{a R^2}{2 || \mu ||_{\HCal}} \cdot || w^*_{\phi} ||_{\HCal} \right) \\
&\geq ||w^*_{\phi} ||_{\HCal} \cdot || \mu ||_{\HCal} \cdot \left( 1 - \frac{a R^2}{2 || \mu ||_{\HCal}} \cdot r \right),
\end{align*}
where the last line is since by definition $|| w^*_{\phi} ||_{\HCal} \leq r$.
Thus, for $\epsilon = \frac{a R^2}{2 || \mu ||_{\HCal}}$, 
$$ \left\langle \frac{w^*_{\phi}}{ ||w^*_{\phi} ||_{\HCal} }, \frac{\mu}{|| \mu ||_{\HCal}} \right\rangle_{\HCal} \geq 1 - \epsilon. $$
It follows that the two unit vectors can be made arbitrarily close to each other by decreasing $r$.
Since this corresponds to increasing the strength of regularisation (by Lagrange duality), and since $\frac{\mu}{|| \mu ||_{\HCal}}$ corresponds to the normalised unhinged solution for any regularisation strength, the result follows.

%
\end{proof}


%
\subsection{Additional helper lemmas}

\begin{lemma}
\label{lemm:sum-constant}
Pick any loss $\ell$.
Suppose that
$$ ( \forall \rhoPlus \in [ 0, 1/2 ) ) \, ( \exists \alpha, \beta \in \Real ) \, ( \forall y, v ) \, \ell(y, v) = \alpha \cdot \ellContam(y, v) + \beta.$$
Then,
$$ ( \exists C \in \Real ) \, ( \forall v \in \Real) \, \ell_1(v) + \ell_{-1}(v) = C. $$
\end{lemma}

\begin{proof}[Proof of Lemma \ref{lemm:sum-constant}]
By the definition of the noise-corrected loss (Equation \ref{eqn:melted-loss}), the given statement is that there exist $\alpha, \beta \colon [0, 1/2) \to \Real$ with
$$ ( \forall \rhoPlus \in [ 0, 1/2 ) ) \, ( \forall v \in \Real ) \, \begin{bmatrix} \ell_1(v) \\ \ell_{-1}(v) \end{bmatrix} = \alpha(\rhoPlus) \cdot \begin{bmatrix} 1 - \rhoPlus & \rhoPlus \\ \rhoPlus & 1 - \rhoPlus \end{bmatrix}^{-1} \cdot \begin{bmatrix} \ell_1(v) \\ \ell_{-1}(v) \end{bmatrix} + \beta(\rhoPlus). $$
Expanding out the matrix inverse,
$$ ( \forall \rhoPlus \in [ 0, 1/2 ) ) \, ( \forall v \in \Real ) \, \begin{bmatrix} \ell_1(v) \\ \ell_{-1}(v) \end{bmatrix} = \frac{\alpha(\rhoPlus)}{1 - 2\rhoPlus} \cdot \begin{bmatrix} 1 - \rhoPlus & -\rhoPlus \\ -\rhoPlus & 1 - \rhoPlus \end{bmatrix} \cdot \begin{bmatrix} \ell_1(v) \\ \ell_{-1}(v) \end{bmatrix} + \beta(\rhoPlus). $$
Adding together the two sets of equations,
$$ ( \forall \rhoPlus \in [ 0, 1/2 ) ) \, ( \forall v \in \Real ) \, \ell_1(v) + \ell_{-1}(v) = \alpha(\rhoPlus) \cdot( \ell_1(v) + \ell_{-1}(v) ) + \beta(\rhoPlus), $$
or
$$ ( \forall \rhoPlus \in [ 0, 1/2 ) ) \, ( \forall v \in \Real ) \, (1 - \alpha(\rhoPlus)) \cdot ( \ell_1(v) + \ell_{-1}(v) ) = \beta(\rhoPlus). $$
Since the RHS is independent of $v$, the LHS cannot depend on $v$, \ie $\ell_1(v) + \ell_{-1}(v)$ must be a constant.
\end{proof}

\clearpage

{\LARGE
\begin{center}
\textbf{Additional Discussion for ``Learning with Symmetric Label Noise: The Importance of Being Unhinged'''}
\end{center}
}

\section{Evidence that non-convex losses and linear scorers may not be SLN-robust}
\label{sec:app-robustness-non-convex}

We now present evidence that for $\ell$ being the TangentBoost loss,
$$ \ell(y, v) = ( 2 \tan^{-1}( y v ) - 1 )^2, $$
or the $t$-logistic regression loss for $t = 2$,
$$ \ell(y, v) = \log( 1 - y v + \sqrt{1 + v^2} ), $$
$( \ell, \LinearClass )$ is not SLN-robust. We do this by looking at the minimisers of these losses on the 2D example of \citet{Long:2010}.
Of course, as these losses are non-convex, exact minimisation of the risk is challenging.
However, as the search space is $\Real^2$, we construct a grid of resolution $0.025$ over $[ -10, 10 ]^2$.
We then exhaustively compute the objective for all grid points, and seek the minimiser.

We apply this procedure to the \citet{Long:2010} dataset with $\gamma = \frac{1}{60}$, and with a $30\%$ noise rate.
Figure \ref{fig:tanboost-surf} plots the results of the objective for the TangentBoost loss.
We find that the minimiser is at $w^* = ( 0.2 , 1.3  )$.
This results in a classifier with error rate of $\frac{1}{2}$ on $\D$.
Similarly, from Figure \ref{fig:tlog-surf}, we find that the minimiser is $w^* = ( 1.025, 5.1 )$, which also results in a classifier with error rate of $\frac{1}{2}$.

\begin{figure}[htb]
	\centering
	
	\includegraphics[scale=0.175]{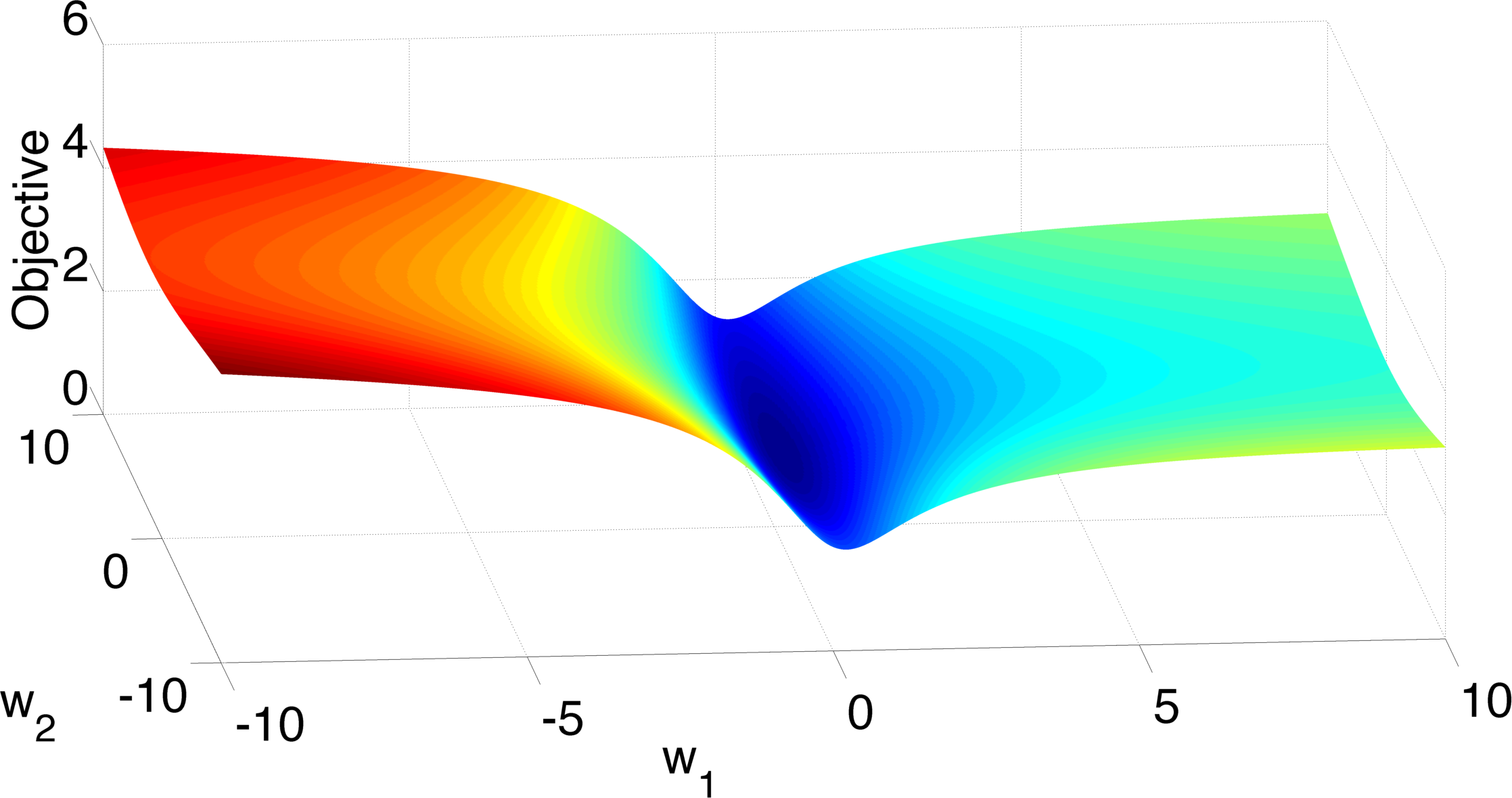}
	
	\caption{Risk values for various weight vectors $w = (w_1, w_2)$, TangentBoost, \citet{Long:2010} dataset.}
	\label{fig:tanboost-surf}
\end{figure}

\begin{figure}[htb]
	\centering
	
	\includegraphics[scale=0.175]{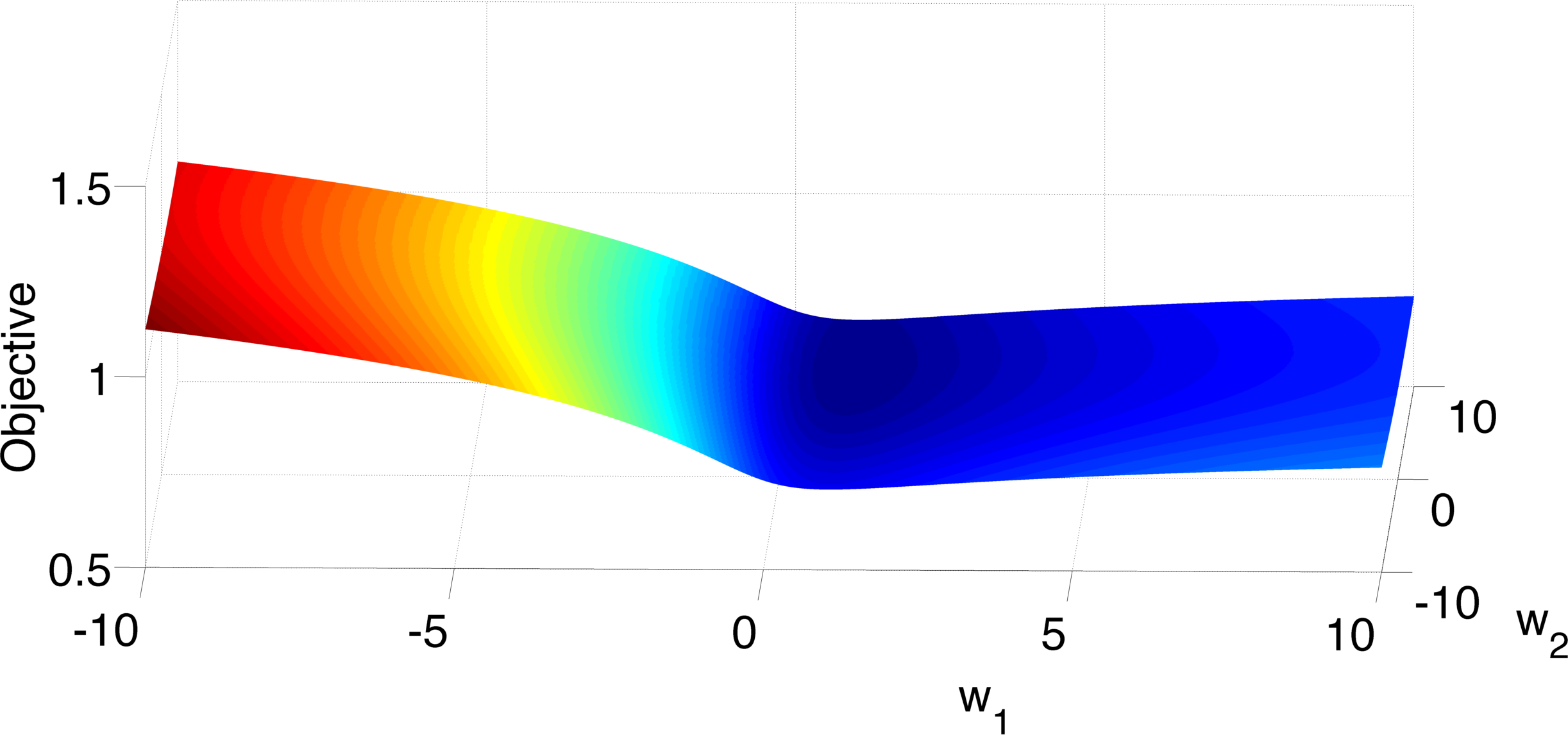}
	
	\caption{Risk values for various weight vectors $w = (w_1, w_2)$, $t$-logistic regression, \citet{Long:2010} dataset.}
	\label{fig:tlog-surf}
\end{figure}

The shape of these plots suggests that the minimiser is indeed found in the interval $[ -10, 10 ]^2$.
To further verify this, we performed L-BFGS minimisation of these losses using $100$ different random initialisations, uniformly from $[ -100, 100 ]^2$.
We find that in each trial, the TangentBoost solution converges to $w^* = ( 0.2122, 1.3031 )$, while the $t$-logistic solution converges to $w^* = ( 1.0372, 5.0873 )$,
both of which result in accuracy of $\frac{1}{2}$ on $\D$.

\subsection{In defence of non-convex losses: beyond SLN-robustness}

The above illustrates the possible non SLN-robustness of two non-convex losses.
However, there may be \emph{other} notions under which these losses are robust.
For example, \citet{Ding:2010} defines robustness to be a stability of the asymptotic maximum likelihood solution when adding a new labelled instance (chosen \emph{arbitrarily} from $\XCal \times \PMOne$), based on a definition in \citet{OHagan:1979}.
Intuitively, this captures robustness to outliers in the instance space, so that \eg an adversarial mislabelling of an instance far from the true decision boundary does not adversely affect the learned model. 
Such a notion of robustness is clearly of practical interest, and future study of such alternate notions would be of value.

\subsection{Conjecture: (most) strictly proper composite losses are not SLN-robust}

More abstractly, we conjecture the above can be generalised to the following.
Recall that a loss $\ell$ is \emph{strictly proper composite} \citep{Reid:2010} if its (unique) Bayes-optimal scorer is some strictly monotone transformation $\psi$ of the class-probability function:
$ ( \forall \D ) \, \BayesOpt{\D}{\ell} = \{ \psi \circ \eta \}. $

\begin{conjecture}
Pick any strictly proper composite (but not necessarily convex) $\ell$ whose link function has range $\Real$. Then, $( \ell, \LinearClass )$ is not SLN-robust.
\end{conjecture}

We believe the above is true for the following reason.
Suppose $\D$ is some linearly separable distribution, with $\eta \colon x \mapsto \indicator{ \langle w^*, x \rangle > 0 }$ for some $w^*$.
Then, minimising $\ell$ with $\LinearClass$ will be well-specified: the Bayes-optimal scorer is $\psi( \indicator{ \langle w^*, x \rangle > 0 } )$.
If the range of $\psi$ is $\Real$, then this is equivalent to $\infty \cdot (2\indicator{ \langle w^*, x \rangle > 0 } - 1)$, which is in $\LinearClass$ if we allow for the extended reals.
The resulting classifier will thus have $100\%$ accuracy.

However, by injecting any non-zero label noise, minimising $\ell$ with $\LinearClass$ will no longer be well-specified, as $\etaContam$ takes on the values $\{ 1 - \rhoPlus, \rhoPlus \}$, which cannot be the sole set of output scores for any linear scorer if $| \XCal | > 3$.
We believe it unlikely that every such misspecified solution have $100\%$ accuracy on $\D$.
We further believe it likely that one can exhibit a scenario, possibly the same as the \citet{Long:2010} example, where the resulting solution has accuracy 50\%.

Two further comments are in order.
First, if a loss is strictly proper composite, then it cannot satisfy Equation \ref{eqn:constant}, and hence it cannot be strongly SLN-robust.
(However, this does leave open the possibility that with $\LinearClass$, the loss is SLN-robust.)
Second, observe that the restriction that $\psi$ have range $\Real$ is necessary to rule out cases such as square loss,
where the link function has range $[ -1, 1 ]$.

\section{Preservation of mean maps}
\label{sec:app-mean-map}

Pick any $\D$, and $\rhoPlus \in [0, 1/2)$.
Then,
\begin{align*}
( \forall x \in \XCal ) \, 2\etaContam(x) - 1 &= 2 \cdot (( 1 - 2\rhoPlus ) \cdot \eta(x) + \rhoPlus) - 1 \\
&= (1 - 2 \rhoPlus) \cdot ( 2\eta(x) - 1 ).
\end{align*}
Thus, for any feature mapping $\Phi \colon \XCal \to \HCal$, the kernel mean map of the clean distribution is
\begin{align*}
\Expectation{ ( \X, \Y ) \sim D }{ \Y \cdot \Phi( \X ) } &= \Expectation{ \X \sim M }{ (2\eta(\X) - 1) \cdot \Phi( \X ) } \\
&= \frac{1}{(1 - 2 \rhoPlus} \cdot \Expectation{ \X \sim M }{ (2\etaContam(\X) - 1) \cdot \Phi( \X ) } \\
&= \frac{1}{(1 - 2\rhoPlus)} \cdot \Expectation{ ( \X, \Y ) \sim \CCNDSymm }{ \Y \cdot \Phi( \X ) },
\end{align*}
which is a scaled version of the kernel mean map of the noisy distribution.
That is, the kernel mean map is preserved under symmetric label noise.
Instantiating the above with a specific instance $x \in \XCal$ gives Equation \ref{eqn:mean-immunity}.

\section{Additional theoretical considerations}
\label{sec:app-theory}

%
\subsection{Generalisation bounds}

Generalisation bounds are readily derived for the unhinged loss.
For a training sample $\mathsf{S} \sim D^n$, define the $\ell$-deviation of a scorer $\scorer$ to be the difference in its population and empirical $\ell$-risk,
$$ \text{dev}^{\D, \mathsf{S}}_{\ell}( s ) = \EllRisk{\D}{s} - \EllRisk{\mathsf{S}}{s}. $$
This quantity is of interest because a standard result says that for the empirical risk minimiser $s_n$ over some function class $\FCal$, $\regret^{\D, \FCal}_{\ell}( s_n ) \leq 2 \cdot \sup_{s \in \FCal} | \text{dev}^{\D, \mathsf{S}}_{\ell}( s ) |$ \citep[Equation 2]{Boucheron:2005}.
For unhinged loss, we have the following Rademacher based bound.

\begin{proposition}
\label{prop:rado}
Pick any $\D$ and $n \in \mathbb{N}_+$.
Let $\mathsf{S} \sim D^{n}$ denote an empirical sample.
For some $B \in \Real_+$, let $s \in \FCal_B$.
Then, with probability at least $1 - \delta$ over the choice of $\mathsf{S}$, for $\ell = \ellUnhinged$,
$$ \mathrm{dev}^{\D, \mathsf{S}}_{\ell}( s ) \leq 2 \cdot \RCal_n( \FCal_B, \mathsf{S} ) + B \cdot \sqrt{\frac{\log \frac{2}{\delta}}{2n}} $$
where $\RCal_n( \FCal_B, \mathsf{S} )$ is the empirical Rademacher complexity of $\FCal_B$ on sample $\mathsf{S}$.
\end{proposition}

\begin{proof}[Proof of Proposition \ref{prop:rado}]
The standard Rademacher-complexity generalisation bound \citep[Theorem 7]{Bartlett:2002}, \citep[Theorem 4.1]{Boucheron:2005} states that with probability at least $1 - \delta$ over the choice of $\mathsf{S}$,
$$ \mathrm{dev}^{\D, \mathsf{S}}_{\ell}( s ) \leq 2 \cdot || (\ell)' ||_\infty \cdot \RCal_n( \FCal_B, \mathsf{S} ) + || \ell ||_{\infty} \cdot \sqrt{\frac{\log \frac{2}{\delta}}{2n}}. $$
For the unhinged loss, $ || (\ellUnhinged)' ||_\infty = 1 $.
Further, since we work over bounded scorers, $|| \ellUnhinged ||_{\infty} = B$.
The result follows.
\end{proof}

Proposition \ref{prop:rado} holds equally when learning from a corrupted sample $\bar{\mathsf{S}} \sim \ContamDShort^n$.
Since $\regret^{\D, \FCal}_{\ellUnhinged}( s_n ) = \frac{1}{1 - 2\rhoPlus} \cdot \regret^{\ContamDShort, \FCal}_{\ellUnhinged}( s_n )$ by Proposition \ref{prop:surrogate-regret}, by minimising the unhinged loss on the corrupted sample, we can bound the regret on the clean distribution.

\section{Additional relations to existing methods}
\label{sec:app-relations}

We discuss some further connections of the unhinged loss to existing methods.

\subsection{Unhinging the SVM}

We can motivate the unhinged loss intuitively by studying the noise-corrected versions of the hinge loss, as per Equation \ref{eqn:melted-loss}.
Figure \ref{fig:slightly-unhinged} shows the noise corrected hinge loss for $\rhoPlus \in \{ 0, 0.2, 0.4 \}$.
We see that as the noise rate increases, the effect is to slightly \emph{unhinge} the original loss, by removing its flat portion\footnote{Another interesting observation is that these noise-corrected losses are negatively unbounded -- that is, minimising hinge loss on $\D$ is equivalent to minimising a negatively unbounded loss on $\ContamDShort$. This is another justification for studying negatively unbounded losses.}.
Thus, if we knew the noise rate  $\rhoPlus$, we could use these \emph{slightly unhinged} losses to learn.

\begin{figure}[htb]
	\centering

\begin{tikzpicture}[scale=0.75]
\pgfplotsset{every axis plot post/.append style={mark=none}}

\begin{axis}[
	legend style={at={(0.9,0.97)},anchor=north west},
	domain=-2:2,
	ultra thick,
	xlabel=$v$,ylabel=$\ell_1( v )$
	]
	\addplot {max(0,1-x)};
	\addplot[dashed, color=red] {(1/0.6)*(0.8*max(0,1-x)-0.2*max(0,1+x))};	
	\addplot[dotted, color=brown] {(1/0.2)*(0.6*max(0,1-x)-0.4*max(0,1+x))};			

	\legend{$\rhoPlus = 0$,$\rhoPlus = 0.2$,$\rhoPlus = 0.4$}
\end{axis}
\end{tikzpicture}

	\caption{Noise-corrected versions of hinge loss, $\ell_1(v) = \max( 0, 1 - v )$. Best viewed in colour.}
	\label{fig:slightly-unhinged}
\end{figure}
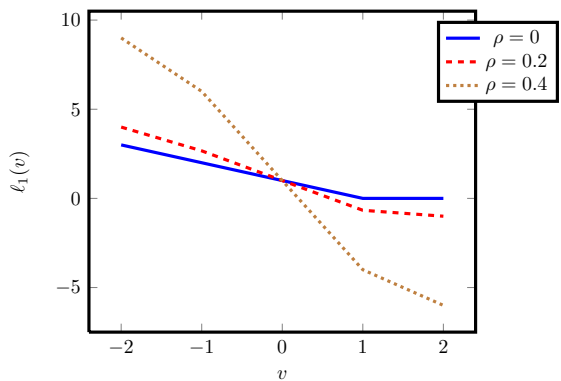

Of course, in general we do not know the noise rate.
Further, the slightly unhinged losses are non-convex.
So, in order to be robust to an \emph{arbitrary} noise rate $\rhoPlus$, we can \emph{completely unhinge} the loss, yielding
\begin{align*}
\ellUnhinged_{1}(v) &= 1 - v \text{ and } \ellUnhinged_{-1}(v) = 1 + v.
\end{align*}

\subsection{Relation to centroid classifiers}

As established in \S\ref{sec:unhinged-centroid}, the optimal unhinged classifier (Equation \ref{eqn:mmd}) is equivalent to a centroid classifier, where one replaces the positive and negative classes by their centroids, and performs classification based on the distance of an instance to the two centroids.
Such a classifier has been proposed as a prototypical example of a simple kernel-based classifier \citep[Section 1.2]{Scholkopf:2002}, \citep[Section 5.1]{Shawe-Taylor:2004}
\citet[Definition 4]{Balcan:2008} considers such classification rules using general similarity functions in place of kernels corresponding to an RKHS.

The optimal unhinged classifier is also closely related to the Rocchio classifier in information retrieval \citep[pg.\ 181]{Manning:2008}, and the nearest centroid classifier in computational genomics \citep{Tibshirani:2002}.
The optimal kernelised scorer for these approaches is \citep{Doloc-Mihu:2003}
$$ s^* \colon x \mapsto \left( \Expectation{\X \sim P}{ k( \X, x ) } - \Expectation{\X \sim Q}{ k( \X, x ) } \right), $$
\ie it does not weight each of the kernel means.

%
\subsection{Relation to kernel density estimation}

When working with an RKHS with a translation invariant kernel\footnote{For a general (not necessarily translation invariant) kernel, this is known as a potential function rule \citep[\S10.3]{Devroye:1996}. The use of ``potential'' here is distinct from that of a ``convex potential''.}, the optimal unhinged scorer (Equation \ref{eqn:mmd}) can be interpreted as follows: perform kernel density estimation on the positive and negative classes, and then classify instances according to Bayes' rule.
For example, with a Gaussian RBF kernel, the classifier is equivalent to using a Gaussian kernel to compute density estimates of $P, Q$, and using these to classify.
This is known as a kernel classification rule \citep[Chapter 10]{Devroye:1996}.

This perspective suggests that in computing $\OptScorer{\unhinged}$, we may also estimate the corrupted class-probability function.
In particular, observe that if we compute
$ \frac{\pi}{1 - \pi} \cdot \frac{ \Expectation{\X \sim P}{ k( \X, x ) } }{ \Expectation{\X \sim Q}{ k( \X, x ) } } $, similar to the Nadaraya-Watson estimator \citep[pg.\ 300]{Bishop:2006}, 
then this provides an estimate of $\frac{\eta(x)}{1 - \eta(x)}$.
Of course, such an approach will succumb to the curse of dimensionality\footnote{This refers to the rate of convergence of the estimate of $\eta$ to the true $\eta$. By contrast, generalisation bounds establish that the rate of convergence of the estimate of the corresponding \emph{classifier} to the Bayes-optimal classifier $\sign( 2\eta(x) - 1 )$ is independent of the dimension of the feature space.}.

An alternative is to use the Probing reduction \citep{Langford:2005}, by computing an ensemble of cost-sensitive classifiers at varying cost ratios.
To this end, observe that the following weighted unhinged (or \emph{whinge}) loss,
\begin{align*}
\ellWhinge_{1}(v) &= c_{1} \cdot -v \\
\ellWhinge_{-1}(v) &= c_{-1} \cdot v
\end{align*}
for some $c_{-1} \in [0, 1]$ and $c_{1} = 1 - c_{-1}$, will have a restricted Bayes-optimal scorer of $B \cdot \sign( \eta(x) - {c_{-1}} )$ over $\BoundedScorers$.
Further, it will result in an optimal scorer that simply weights each of the kernel means,
$$ s^*_{\whinge, \lambda} \colon x \mapsto \frac{1}{\lambda} \cdot \Expectation{( \X, \Y ) \sim \D}{ c_{\Y} \cdot \Y \cdot k( \X, x ) }, $$
making it trivial to compute as $c$ is varied.

%
\subsection{Relation to the MMD witness}
\label{sec:unhinged-mmd}

The optimal weight vector for unhinged loss (Equation \ref{eqn:optimal-weight}) can be expressed as
$$ \OptWeight{\unhinged} = \frac{1}{\lambda} \cdot ( {\pi} \cdot \meanMapPos - ( 1 - {\pi} ) \cdot \meanMapNeg), $$
where $\meanMapPos$ and $\meanMapNeg$ are the \emph{kernel mean maps} with respect to $\HCal$ of the positive and negative class-conditionals distributions,
\begin{align*}
\meanMapPos &= \Expectation{\X \sim P}{ \Phi( \X ) } \\
\meanMapNeg &= \Expectation{\X \sim Q}{ \Phi( \X ) }.
\end{align*}
When $\pi = \frac{1}{2}$, $|| w^*_{1} ||_{\HCal}$ is precisely the \emph{maximum mean discrepancy} (\emph{MMD}) \citep{Gretton:2012} between $P$ and $Q$, using all functions in the unit ball of $\HCal$.
The mapping $x \mapsto \langle w^*_1, x \rangle_{\HCal}$ itself is referred to as the \emph{witness function} \citep[\S2.3]{Gretton:2012}.
While the motivation of MMD is to perform hypothesis testing so as to distinguish between two distributions $P, Q$, 
rather than constructing a suitable scorer, the fact that it arises from the optimal scorer for the unhinged loss has been previously noted \citep[Theorem 1]{Sriperumbudur:2009}.


\section{Example of poor classification with square loss}
\label{sec:app-breaking-the-square}

We illustrate that square loss with a linear function class may perform poorly even when the underlying distribution is linearly separable.
We consider the dataset of \citet{Long:2010}, with \emph{no} label noise.
That is, we have $\XCal = \{ ( 1, 0 ), ( \gamma, 5 \gamma ), ( \gamma, -\gamma ), ( \gamma, -\gamma ) \} \subset \Real^{2}$, and $\eta \colon x \mapsto 1$.
Let $X \in \Real^{4 \times 2}$ be the feature matrix of the four data points.
Then, the optimal weight vector learned by square loss is
\begin{align*}
w^* &= (X^T X)^{-1} X^T \begin{bmatrix} 1 \\ 1 \\ 1 \\ 1 \end{bmatrix} \\
&= \begin{bmatrix} \frac{8\gamma+3}{8\gamma^2+3} \\ -\frac{\gamma+1}{3\gamma \cdot (8\gamma^2+3)} \end{bmatrix}.
\end{align*}
It is easy to check that the predicted scores are then
$$ s^* = \begin{bmatrix} \frac{8\gamma+3}{8\gamma^2+3} \\ \frac{\gamma \cdot (8\gamma+3)}{8\gamma^2+3-\frac{5\cdot(\gamma-1)\gamma}{24\gamma^3+9\gamma}}  \\ \frac{(\gamma-1) \cdot \gamma}{24\gamma^3+9\gamma+\frac{\gamma \cdot (8\gamma+3)}{8\gamma^2+3}} \\ \frac{(\gamma-1) \cdot \gamma}{24\gamma^3+9\gamma+\frac{\gamma \cdot (8\gamma+3)}{8\gamma^2+3}} \end{bmatrix}. $$
But for $\gamma < \frac{1}{12}$, this means that the predicted scores for the last two examples are negative.
That is, the resulting classifier will have $50\%$ accuracy.
(This does not contradict the robustness of square loss, as robustness simply requires that performance is the \emph{same} with and without noise.)

It is initially surprising that square loss fails in this example, as we are employing a linear function class, and the true $\eta$ is expressible as a linear function.
However, recall that the Bayes-optimal scorer for square loss is
$$ \BayesOpt{\D}{\ell} = \{ s \colon x \mapsto 2 \eta(x) - 1 \}. $$
In this case, the Bayes-optimal scorer is
$$ s^* \colon x \mapsto 2 \indicator{ x_1 > 0 } - 1. $$
The application of a threshold means the that scorer is \emph{not expressible as a linear model}.
Therefore, the combination of loss and function class is in fact \emph{not} well-specified for the problem.

To clarify this point, consider the use of the squared hinge loss, $\ell(y, v) = \max( 0, 1 - y v )^2$.
This loss induces a \emph{set} of Bayes-optimal scorers, which are:
$$ \BayesOpt{\D}{\ell} = \left\{ s \mid ( \forall x \in \XCal ) \, \begin{cases} \eta(x) = 1 &\implies s( x ) \in [ 1, \infty ) \\ \eta(x) \in (0, 1) &\implies s( x ) = 2 \eta(x) - 1 \\ \eta(x) = 0 &\implies s( x ) \in (-\infty, 1]. \end{cases} \right\} $$
Crucially, we \emph{can} find a linear scorer that is in this set: for, say, $v = ( \frac{1}{\gamma}, 0 )$, we clearly have $\langle v, x \rangle \geq 1$ for every $x \in \XCal$, and so this is a Bayes-optimal scorer.
Thus, minimising the square hinge loss on this distribution will indeed find a classifier with $100\%$ accuracy.

\section{Example of poor classification with unhinged loss}
\label{sec:app-breaking-the-unhinged}

We illustrate that the unhinged loss with a linear function class may perform poorly even when the underlying distribution is linearly separable.
(For another example where instances are on the unit ball, see \citet[Figure 1]{Balcan:2008}.)
Consider a distribution $\DMN$ uniformly concentrated on $\XCal = \{ x_1, x_2, x_3 \}$ with $x_1 = ( 1, 2 ), x_2 = ( 1, -4 ), x_3 = ( -1, 1 )$, with $\eta(x_1) = \eta(x_2) = 1$ and $\eta(x_3) = 0$, \ie the first two instances are positive, and the third instance negative.
Then it is evident that the optimal unhinged hyperplane, with regularisation strength 1, is $w^* = ( 1, -1 )$.
This will misclassify the first instance as being negative.
Figure \ref{fig:bad-unhinged} illustrates.

It is easy to check that for this particular distribution, the optimal weight for square loss is $w^* = (1, 0)$.
This results in perfect classification.
Thus, we have a reversal of the scenario of the previous section -- here, square loss classifies perfectly, while the unhinged loss classifies no better than random guessing.

It may appear that the above contradicts the classification-calibration of the unhinged loss:
there certainly is a linear scorer that is Bayes-optimal over $\BoundedScorers$, namely, $w^* = ( B, 0 )$.
The subtlety is that in this case, minimisation over the unit ball $|| w ||_2 \leq 1$ (as implied by $\ell_2$ regularisation) is unable to restrict attention to the desired scorer.

There are two ways to rectify examples such as the above.
First, as in general, we can employ a suitably rich kernel, \eg a Gaussian RBF kernel.
It is not hard to verify that on this dataset, such a kernel will find a perfect classifier.
Second, we can look to explicitly {enforce} that minimisation is over all $w$ satisfying $ | \langle w, x_n \rangle | \leq 1 $.
This will result in a linear program (LP) that may be solved easily, but does not admit a closed form solution as in the case of minimising over the unit ball.
It may be checked that the resulting LP will recover the optimal weight $w^* = ( 1, 0 )$.
While this approach is suitable for this particular example, issues arise when dealing with infinite dimensional feature mappings (as we lose the existence of a representer theorem without regularisation based on the norm in the Hilbert space \citep{Yu:2013}).

\begin{figure}

	\centering
	
\begin{tikzpicture}[thick, scale=3]
  
  \coordinate (O) at (0,0);
  \draw[->] (-0.3,0) -- (2,0) coordinate[label = {below:$x_1$}] (xmax);
  \draw[->] (0,-1) -- (0,1) coordinate[label = {right:$x_2$}] (ymax);

   \foreach \Point in {(0.25,0.5),(0.25,-1)}{
   	\node at \Point {$\textbf{\LARGE +}$};
    }
    
   \foreach \Point in {(-0.25,0.25)}{
   	\node at \Point {$\textbf{\LARGE --}$};
    }    

    \draw[very thick, color=red] (-0.75,-0.75)--(0.75,0.75);
		
\end{tikzpicture}
	
	\caption{Example of linearly separable distribution where, when learning with the unhinged loss and a linear function class, the resulting hyperplane (in red) misclassifies one of the instances.}
	\label{fig:bad-unhinged}

\end{figure}
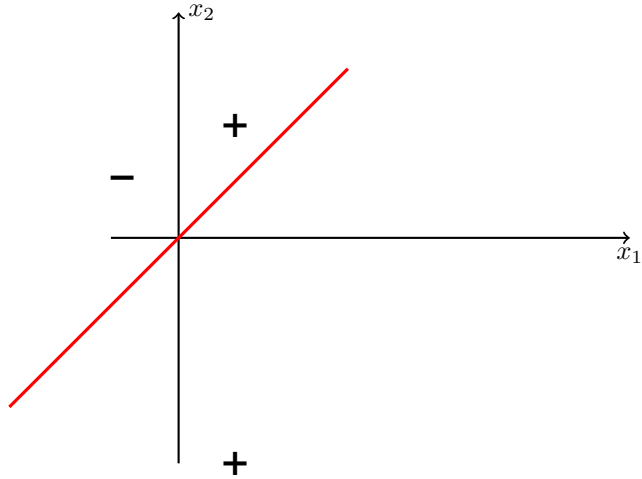

\clearpage

{\LARGE
\begin{center}
\textbf{Additional Experiments for ``Learning with Symmetric Label Noise: The Importance of Being Unhinged''}
\end{center}
}

\section{Additional experimental results}
\label{sec:app-expts}

Table \ref{tbl:full-long-matlab} reports the 0-1 error for a range of losses on the \citet{Long:2010} dataset.
TanBoost refers to the loss of \citet{Hamed:2010}.
As before, we find the unhinged loss to generally find a good classifier.
Observe that the relatively poor performance of the square and TanBoost loss can be attributed to the findings of Appendix \ref{sec:app-robustness-non-convex}, \ref{sec:app-breaking-the-square}.

We next report the 0-1 error and one minus the AUC for a range of datasets.
We begin with a dataset of \citet{Mease:2008}, where $\XCal = [0, 1]^{20}$, and $M$ is the uniform distribution.
Further, we have $\eta \colon x \mapsto \indicator{\langle w^*, x \rangle > 2.5}$ for $w^* = \begin{bmatrix} \mathbf{1}_5 & \mathbf{0}_{15} \end{bmatrix}$, \ie there is a sparse separating hyperplane.
Table \ref{tbl:full-mease-matlab} reports the results on this dataset injected with various levels of symmetric noise.
On this dataset, the $t$-logistic loss generally performs the best.

Finally, we report the 0-1 error and one minus the AUC on some UCI datasets in Tables \ref{tbl:full-iris-matlab} -- \ref{tbl:full-ionosphere-matlab}.
Table \ref{tbl:uci} summarises statistics of the UCI data.
Several datasets are imbalanced, meaning that 0-1 error is not the ideal measure of performance (as it can be made small with a trivial majority classifier).
The AUC is thus arguably a better indication of performance for these datasets.
We generally find that at high noise rates (40\%), the AUC of the unhinged loss is superior to that of other losses.

\begin{table}[!h]

	\centering
	\renewcommand{\arraystretch}{1.5}

	\arrayrulecolor{tabgrey}

	\begin{tabular}{llll}
	
		\toprule
		\toprule
		\textbf{Dataset} & \textbf{$N$} & $D$ & $\mathbb{P}(\Y = 1)$ \\
		\toprule
		Iris & 150 & 4 & 0.3333 \\
		Ionosphere & 351 & 34 & 0.3590 \\		
		Housing & 506 & 13 & 0.0692 \\		
		Car & 1,728 & 8 & 0.0376 \\		
		USPS 0v7 & 2,200 & 256 & 0.5000 \\
		Splice & 3,190 & 61 & 0.2404 \\
		Spambase & 4,601 & 57 & 0.3940 \\		
		\bottomrule

	\end{tabular}
	
	\caption{Summary of UCI datasets. Here, $N$ denotes the total number of samples, and $D$ the dimensionality of the feature space.}
	\label{tbl:uci}
	
\end{table}

\begin{table}[t]
	\centering
	\renewcommand{\arraystretch}{1.25}

{
\begin{tabular}{@{}lllllll@{}}
\toprule
\toprule
& \textbf{Hinge} & \textbf{Logistic} & \textbf{Square} & \textbf{$t$-logistic} & \textbf{TanBoost} & \textbf{Unhinged} \\ 
\midrule
$\rhoPlus = 0$ & \cellcolor{gray!25}{0.00 $\pm$ 0.00} & \cellcolor{gray!25}{0.00 $\pm$ 0.00} & 0.25 $\pm$ 0.00 & \cellcolor{gray!25}{0.00 $\pm$ 0.00} & 0.25 $\pm$ 0.00 & \cellcolor{gray!25}{0.00 $\pm$ 0.00} \\ 
$\rhoPlus = 0.1$ & 0.15 $\pm$ 0.27 & 0.24 $\pm$ 0.05 & 0.25 $\pm$ 0.00 & \cellcolor{gray!25}{0.00 $\pm$ 0.00} & 0.25 $\pm$ 0.00 & \cellcolor{gray!25}{0.00 $\pm$ 0.00} \\ 
$\rhoPlus = 0.2$ & 0.21 $\pm$ 0.30 & 0.25 $\pm$ 0.00 & 0.25 $\pm$ 0.00 & \cellcolor{gray!25}{0.00 $\pm$ 0.00} & 0.25 $\pm$ 0.00 & \cellcolor{gray!25}{0.00 $\pm$ 0.00} \\ 
$\rhoPlus = 0.3$ & 0.38 $\pm$ 0.37 & 0.25 $\pm$ 0.03 & 0.25 $\pm$ 0.02 & 0.22 $\pm$ 0.08 & 0.25 $\pm$ 0.03 & \cellcolor{gray!25}{0.00 $\pm$ 0.00} \\ 
$\rhoPlus = 0.4$ & 0.42 $\pm$ 0.36 & 0.22 $\pm$ 0.08 & 0.22 $\pm$ 0.08 & 0.22 $\pm$ 0.08 & 0.22 $\pm$ 0.08 & \cellcolor{gray!25}{0.00 $\pm$ 0.00} \\ 
$\rhoPlus = 0.49$ & 0.46 $\pm$ 0.38 & 0.39 $\pm$ 0.23 & 0.39 $\pm$ 0.23 & 0.39 $\pm$ 0.23 & 0.39 $\pm$ 0.23 & \cellcolor{gray!25}{0.34 $\pm$ 0.48} \\ 
\midrule
\end{tabular}
}

\caption{Results on \citet{Long:2010} dataset. Reported is the mean and standard deviation of the 0-1 error over 125 trials. Grayed cells denote the best performer at that noise rate.}

	\label{tbl:full-long-matlab}
	
\end{table}

\begin{table}[t]
	\centering
	\renewcommand{\arraystretch}{1.25}
	
\subfloat[0-1 Error.]
{
\begin{tabular}{@{}lllllll@{}}
\toprule
\toprule
& \textbf{Hinge} & \textbf{Logistic} & \textbf{Square} & \textbf{$t$-logistic} & \textbf{TanBoost} & \textbf{Unhinged} \\ 
\midrule
$\rhoPlus = 0$ & 0.02 $\pm$ 0.00 & \cellcolor{gray!25}{0.01 $\pm$ 0.00} & 0.03 $\pm$ 0.00 & \cellcolor{gray!25}{0.01 $\pm$ 0.00} & 0.02 $\pm$ 0.00 & 0.05 $\pm$ 0.00 \\ 
$\rhoPlus = 0.1$ & 0.13 $\pm$ 0.01 & 0.05 $\pm$ 0.01 & 0.06 $\pm$ 0.01 & \cellcolor{gray!25}{0.03 $\pm$ 0.01} & 0.05 $\pm$ 0.01 & 0.06 $\pm$ 0.01 \\ 
$\rhoPlus = 0.2$ & 0.14 $\pm$ 0.01 & 0.09 $\pm$ 0.02 & 0.09 $\pm$ 0.02 & \cellcolor{gray!25}{0.06 $\pm$ 0.02} & 0.08 $\pm$ 0.02 & 0.08 $\pm$ 0.02 \\ 
$\rhoPlus = 0.3$ & 0.15 $\pm$ 0.01 & 0.13 $\pm$ 0.03 & 0.13 $\pm$ 0.03 & \cellcolor{gray!25}{0.12 $\pm$ 0.03} & \cellcolor{gray!25}{0.12 $\pm$ 0.03} & \cellcolor{gray!25}{0.12 $\pm$ 0.02} \\ 
$\rhoPlus = 0.4$ & \cellcolor{gray!25}{0.17 $\pm$ 0.05} & 0.24 $\pm$ 0.08 & 0.24 $\pm$ 0.08 & 0.23 $\pm$ 0.07 & 0.23 $\pm$ 0.08 & 0.23 $\pm$ 0.08 \\ 
$\rhoPlus = 0.49$ & 0.47 $\pm$ 0.24 & \cellcolor{gray!25}{0.46 $\pm$ 0.11} & 0.47 $\pm$ 0.11 & 0.48 $\pm$ 0.10 & 0.47 $\pm$ 0.12 & 0.48 $\pm$ 0.12 \\ 
\midrule
\end{tabular}
}

\subfloat[1 - AUC.]{
\begin{tabular}{@{}lllllll@{}}
\toprule
\toprule
& \textbf{Hinge} & \textbf{Logistic} & \textbf{Square} & \textbf{$t$-logistic} & \textbf{TanBoost} & \textbf{Unhinged} \\ 
\midrule
$\rhoPlus = 0$ & \cellcolor{gray!25}{0.00 $\pm$ 0.00} & \cellcolor{gray!25}{0.00 $\pm$ 0.00} & 0.01 $\pm$ 0.00 & \cellcolor{gray!25}{0.00 $\pm$ 0.00} & \cellcolor{gray!25}{0.00 $\pm$ 0.00} & 0.01 $\pm$ 0.00 \\ 
$\rhoPlus = 0.1$ & 0.25 $\pm$ 0.10 & 0.02 $\pm$ 0.01 & 0.02 $\pm$ 0.01 & \cellcolor{gray!25}{0.00 $\pm$ 0.00} & 0.02 $\pm$ 0.01 & 0.02 $\pm$ 0.01 \\ 
$\rhoPlus = 0.2$ & 0.34 $\pm$ 0.10 & 0.05 $\pm$ 0.02 & 0.05 $\pm$ 0.02 & \cellcolor{gray!25}{0.02 $\pm$ 0.01} & 0.04 $\pm$ 0.02 & 0.05 $\pm$ 0.02 \\ 
$\rhoPlus = 0.3$ & 0.41 $\pm$ 0.11 & 0.11 $\pm$ 0.04 & 0.11 $\pm$ 0.04 & \cellcolor{gray!25}{0.09 $\pm$ 0.04} & 0.11 $\pm$ 0.04 & 0.10 $\pm$ 0.04 \\ 
$\rhoPlus = 0.4$ & 0.44 $\pm$ 0.12 & 0.24 $\pm$ 0.08 & 0.24 $\pm$ 0.08 & 0.24 $\pm$ 0.08 & 0.24 $\pm$ 0.08 & \cellcolor{gray!25}{0.23 $\pm$ 0.08} \\ 
$\rhoPlus = 0.49$ & 0.50 $\pm$ 0.13 & 0.47 $\pm$ 0.11 & 0.47 $\pm$ 0.11 & 0.47 $\pm$ 0.11 & 0.47 $\pm$ 0.11 & \cellcolor{gray!25}{0.46 $\pm$ 0.11} \\ 
\midrule
\end{tabular}
}

	\caption{Results on \texttt{mease} dataset. Reported is the mean and standard deviation of performance over 125 trials. Grayed cells denote the best performer at that noise rate.}
	\label{tbl:full-mease-matlab}
	
\end{table}

\begin{table}[t]
	\centering
	\renewcommand{\arraystretch}{1.25}

\subfloat[0-1 Error.]{
\begin{tabular}{@{}lllllll@{}}
\toprule
\toprule
& \textbf{Hinge} & \textbf{Logistic} & \textbf{Square} & \textbf{$t$-logistic} & \textbf{TanBoost} & \textbf{Unhinged} \\ 
\midrule
$\rhoPlus = 0$ & \cellcolor{gray!25}{0.00 $\pm$ 0.00} & \cellcolor{gray!25}{0.00 $\pm$ 0.00} & \cellcolor{gray!25}{0.00 $\pm$ 0.00} & \cellcolor{gray!25}{0.00 $\pm$ 0.00} & \cellcolor{gray!25}{0.00 $\pm$ 0.00} & \cellcolor{gray!25}{0.00 $\pm$ 0.00} \\ 
$\rhoPlus = 0.1$ & 0.01 $\pm$ 0.03 & 0.01 $\pm$ 0.01 & 0.01 $\pm$ 0.02 & 0.01 $\pm$ 0.03 & 0.01 $\pm$ 0.02 & \cellcolor{gray!25}{0.00 $\pm$ 0.00} \\ 
$\rhoPlus = 0.2$ & 0.06 $\pm$ 0.12 & 0.02 $\pm$ 0.05 & 0.03 $\pm$ 0.04 & 0.04 $\pm$ 0.05 & 0.03 $\pm$ 0.05 & \cellcolor{gray!25}{0.00 $\pm$ 0.01} \\ 
$\rhoPlus = 0.3$ & 0.17 $\pm$ 0.20 & 0.09 $\pm$ 0.10 & 0.08 $\pm$ 0.09 & 0.09 $\pm$ 0.11 & 0.09 $\pm$ 0.10 & \cellcolor{gray!25}{0.02 $\pm$ 0.07} \\ 
$\rhoPlus = 0.4$ & 0.35 $\pm$ 0.24 & 0.24 $\pm$ 0.17 & 0.24 $\pm$ 0.17 & 0.24 $\pm$ 0.16 & 0.24 $\pm$ 0.17 & \cellcolor{gray!25}{0.13 $\pm$ 0.22} \\ 
$\rhoPlus = 0.49$ & 0.60 $\pm$ 0.20 & 0.49 $\pm$ 0.20 & 0.49 $\pm$ 0.19 & 0.49 $\pm$ 0.20 & 0.49 $\pm$ 0.19 & \cellcolor{gray!25}{0.45 $\pm$ 0.33} \\ 
\midrule
\end{tabular}
}

\subfloat[1 - AUC.]{
\begin{tabular}{@{}lllllll@{}}
\toprule
\toprule
& \textbf{Hinge} & \textbf{Logistic} & \textbf{Square} & \textbf{$t$-logistic} & \textbf{TanBoost} & \textbf{Unhinged} \\ 
\midrule
$\rhoPlus = 0$ & \cellcolor{gray!25}{0.00 $\pm$ 0.00} & \cellcolor{gray!25}{0.00 $\pm$ 0.00} & \cellcolor{gray!25}{0.00 $\pm$ 0.00} & \cellcolor{gray!25}{0.00 $\pm$ 0.00} & \cellcolor{gray!25}{0.00 $\pm$ 0.00} & \cellcolor{gray!25}{0.00 $\pm$ 0.00} \\ 
$\rhoPlus = 0.1$ & \cellcolor{gray!25}{0.00 $\pm$ 0.00} & \cellcolor{gray!25}{0.00 $\pm$ 0.00} & \cellcolor{gray!25}{0.00 $\pm$ 0.00} & \cellcolor{gray!25}{0.00 $\pm$ 0.00} & \cellcolor{gray!25}{0.00 $\pm$ 0.00} & \cellcolor{gray!25}{0.00 $\pm$ 0.00} \\ 
$\rhoPlus = 0.2$ & 0.03 $\pm$ 0.11 & \cellcolor{gray!25}{0.00 $\pm$ 0.01} & \cellcolor{gray!25}{0.00 $\pm$ 0.00} & \cellcolor{gray!25}{0.00 $\pm$ 0.01} & \cellcolor{gray!25}{0.00 $\pm$ 0.01} & \cellcolor{gray!25}{0.00 $\pm$ 0.00} \\ 
$\rhoPlus = 0.3$ & 0.14 $\pm$ 0.26 & 0.02 $\pm$ 0.06 & 0.02 $\pm$ 0.05 & 0.02 $\pm$ 0.06 & 0.02 $\pm$ 0.05 & \cellcolor{gray!25}{0.01 $\pm$ 0.06} \\ 
$\rhoPlus = 0.4$ & 0.36 $\pm$ 0.38 & 0.13 $\pm$ 0.18 & 0.13 $\pm$ 0.18 & 0.14 $\pm$ 0.18 & 0.13 $\pm$ 0.18 & \cellcolor{gray!25}{0.09 $\pm$ 0.27} \\ 
$\rhoPlus = 0.49$ & 0.72 $\pm$ 0.34 & 0.47 $\pm$ 0.31 & 0.48 $\pm$ 0.30 & 0.48 $\pm$ 0.30 & 0.48 $\pm$ 0.30 & \cellcolor{gray!25}{0.45 $\pm$ 0.48} \\ 
\midrule
\end{tabular}
}

\caption{Results on \texttt{iris} dataset. Reported is the mean and standard deviation of performance over 125 trials. Grayed cells denote the best performer at that noise rate.}

\label{tbl:full-iris-matlab}
	
\end{table}

\begin{table}[t]
	\centering
	\renewcommand{\arraystretch}{1.25}

\subfloat[0-1 Error.]{
\begin{tabular}{@{}lllllll@{}}
\toprule
\toprule
& \textbf{Hinge} & \textbf{Logistic} & \textbf{Square} & \textbf{$t$-logistic} & \textbf{TanBoost} & \textbf{Unhinged} \\ 
\midrule
$\rhoPlus = 0$ & \cellcolor{gray!25}{0.11 $\pm$ 0.00} & 0.13 $\pm$ 0.00 & 0.17 $\pm$ 0.00 & 0.24 $\pm$ 0.00 & 0.17 $\pm$ 0.00 & 0.20 $\pm$ 0.00 \\ 
$\rhoPlus = 0.1$ & 0.17 $\pm$ 0.04 & 0.18 $\pm$ 0.04 & \cellcolor{gray!25}{0.16 $\pm$ 0.03} & 0.19 $\pm$ 0.05 & 0.17 $\pm$ 0.04 & 0.19 $\pm$ 0.02 \\ 
$\rhoPlus = 0.2$ & 0.20 $\pm$ 0.05 & 0.19 $\pm$ 0.05 & \cellcolor{gray!25}{0.18 $\pm$ 0.04} & 0.21 $\pm$ 0.06 & \cellcolor{gray!25}{0.18 $\pm$ 0.04} & 0.19 $\pm$ 0.02 \\ 
$\rhoPlus = 0.3$ & 0.23 $\pm$ 0.06 & 0.22 $\pm$ 0.05 & 0.22 $\pm$ 0.05 & 0.24 $\pm$ 0.06 & 0.22 $\pm$ 0.05 & \cellcolor{gray!25}{0.21 $\pm$ 0.03} \\ 
$\rhoPlus = 0.4$ & 0.31 $\pm$ 0.11 & 0.31 $\pm$ 0.10 & 0.29 $\pm$ 0.09 & 0.32 $\pm$ 0.09 & 0.30 $\pm$ 0.10 & \cellcolor{gray!25}{0.27 $\pm$ 0.12} \\ 
$\rhoPlus = 0.49$ & 0.48 $\pm$ 0.16 & 0.47 $\pm$ 0.16 & 0.47 $\pm$ 0.16 & 0.47 $\pm$ 0.14 & \cellcolor{gray!25}{0.45 $\pm$ 0.15} & 0.46 $\pm$ 0.22 \\ 
\midrule
\end{tabular}
}

\subfloat[1 - AUC.]{
\begin{tabular}{@{}lllllll@{}}
\toprule
\toprule
& \textbf{Hinge} & \textbf{Logistic} & \textbf{Square} & \textbf{$t$-logistic} & \textbf{TanBoost} & \textbf{Unhinged} \\ 
\midrule
$\rhoPlus = 0$ & 0.12 $\pm$ 0.00 & 0.13 $\pm$ 0.00 & \cellcolor{gray!25}{0.07 $\pm$ 0.00} & 0.20 $\pm$ 0.00 & \cellcolor{gray!25}{0.07 $\pm$ 0.00} & 0.21 $\pm$ 0.00 \\ 
$\rhoPlus = 0.1$ & 0.18 $\pm$ 0.07 & 0.18 $\pm$ 0.07 & \cellcolor{gray!25}{0.12 $\pm$ 0.04} & 0.22 $\pm$ 0.07 & 0.13 $\pm$ 0.05 & 0.21 $\pm$ 0.00 \\ 
$\rhoPlus = 0.2$ & 0.23 $\pm$ 0.09 & 0.22 $\pm$ 0.09 & \cellcolor{gray!25}{0.18 $\pm$ 0.07} & 0.25 $\pm$ 0.08 & 0.19 $\pm$ 0.08 & 0.21 $\pm$ 0.01 \\ 
$\rhoPlus = 0.3$ & 0.31 $\pm$ 0.11 & 0.29 $\pm$ 0.09 & 0.26 $\pm$ 0.09 & 0.30 $\pm$ 0.09 & 0.27 $\pm$ 0.09 & \cellcolor{gray!25}{0.21 $\pm$ 0.01} \\ 
$\rhoPlus = 0.4$ & 0.40 $\pm$ 0.11 & 0.40 $\pm$ 0.10 & 0.38 $\pm$ 0.10 & 0.40 $\pm$ 0.10 & 0.38 $\pm$ 0.10 & \cellcolor{gray!25}{0.25 $\pm$ 0.12} \\ 
$\rhoPlus = 0.49$ & 0.49 $\pm$ 0.12 & 0.50 $\pm$ 0.10 & 0.50 $\pm$ 0.10 & 0.50 $\pm$ 0.10 & 0.50 $\pm$ 0.10 & \cellcolor{gray!25}{0.46 $\pm$ 0.25} \\ 
\midrule
\end{tabular}
}

\caption{Results on \texttt{ionosphere} dataset. Reported is the mean and standard deviation of performance over 125 trials. Grayed cells denote the best performer at that noise rate.}

\label{tbl:full-ionosphere-matlab}

\end{table}

\begin{table}[t]
	\centering
	\renewcommand{\arraystretch}{1.25}
	
\subfloat[0-1 Error.]{
\begin{tabular}{@{}lllllll@{}}
\toprule
\toprule
& \textbf{Hinge} & \textbf{Logistic} & \textbf{Square} & \textbf{$t$-logistic} & \textbf{TanBoost} & \textbf{Unhinged} \\ 
\midrule
$\rhoPlus = 0$ & \cellcolor{gray!25}{0.05 $\pm$ 0.00} & \cellcolor{gray!25}{0.05 $\pm$ 0.00} & 0.07 $\pm$ 0.00 & \cellcolor{gray!25}{0.05 $\pm$ 0.00} & 0.07 $\pm$ 0.00 & \cellcolor{gray!25}{0.05 $\pm$ 0.00} \\ 
$\rhoPlus = 0.1$ & 0.06 $\pm$ 0.01 & 0.06 $\pm$ 0.02 & 0.07 $\pm$ 0.02 & 0.07 $\pm$ 0.02 & 0.07 $\pm$ 0.02 & \cellcolor{gray!25}{0.05 $\pm$ 0.00} \\ 
$\rhoPlus = 0.2$ & 0.06 $\pm$ 0.01 & 0.07 $\pm$ 0.03 & 0.07 $\pm$ 0.02 & 0.08 $\pm$ 0.03 & 0.07 $\pm$ 0.02 & \cellcolor{gray!25}{0.05 $\pm$ 0.00} \\ 
$\rhoPlus = 0.3$ & 0.08 $\pm$ 0.04 & 0.10 $\pm$ 0.06 & 0.11 $\pm$ 0.06 & 0.11 $\pm$ 0.05 & 0.11 $\pm$ 0.06 & \cellcolor{gray!25}{0.05 $\pm$ 0.01} \\ 
$\rhoPlus = 0.4$ & 0.14 $\pm$ 0.10 & 0.21 $\pm$ 0.12 & 0.22 $\pm$ 0.12 & 0.24 $\pm$ 0.13 & 0.22 $\pm$ 0.13 & \cellcolor{gray!25}{0.09 $\pm$ 0.10} \\ 
$\rhoPlus = 0.49$ & \cellcolor{gray!25}{0.45 $\pm$ 0.26} & 0.49 $\pm$ 0.16 & 0.50 $\pm$ 0.16 & 0.49 $\pm$ 0.16 & 0.51 $\pm$ 0.17 & 0.46 $\pm$ 0.30 \\ 
\midrule
\end{tabular}
}

\subfloat[1 - AUC.]{
\begin{tabular}{@{}lllllll@{}}
\toprule
\toprule
& \textbf{Hinge} & \textbf{Logistic} & \textbf{Square} & \textbf{$t$-logistic} & \textbf{TanBoost} & \textbf{Unhinged} \\ 
\midrule
$\rhoPlus = 0$ & 0.25 $\pm$ 0.00 & \cellcolor{gray!25}{0.15 $\pm$ 0.00} & 0.17 $\pm$ 0.00 & 0.25 $\pm$ 0.00 & 0.17 $\pm$ 0.00 & 0.69 $\pm$ 0.00 \\ 
$\rhoPlus = 0.1$ & 0.38 $\pm$ 0.12 & \cellcolor{gray!25}{0.27 $\pm$ 0.07} & \cellcolor{gray!25}{0.27 $\pm$ 0.07} & 0.30 $\pm$ 0.09 & \cellcolor{gray!25}{0.27 $\pm$ 0.07} & 0.69 $\pm$ 0.00 \\ 
$\rhoPlus = 0.2$ & 0.41 $\pm$ 0.13 & \cellcolor{gray!25}{0.35 $\pm$ 0.10} & \cellcolor{gray!25}{0.35 $\pm$ 0.10} & \cellcolor{gray!25}{0.35 $\pm$ 0.10} & \cellcolor{gray!25}{0.35 $\pm$ 0.10} & 0.68 $\pm$ 0.00 \\ 
$\rhoPlus = 0.3$ & 0.44 $\pm$ 0.12 & \cellcolor{gray!25}{0.40 $\pm$ 0.11} & \cellcolor{gray!25}{0.40 $\pm$ 0.11} & \cellcolor{gray!25}{0.40 $\pm$ 0.11} & \cellcolor{gray!25}{0.40 $\pm$ 0.11} & 0.69 $\pm$ 0.01 \\ 
$\rhoPlus = 0.4$ & \cellcolor{gray!25}{0.43 $\pm$ 0.12} & 0.45 $\pm$ 0.12 & 0.45 $\pm$ 0.12 & 0.45 $\pm$ 0.12 & 0.45 $\pm$ 0.12 & 0.68 $\pm$ 0.02 \\ 
$\rhoPlus = 0.49$ & \cellcolor{gray!25}{0.45 $\pm$ 0.13} & 0.49 $\pm$ 0.13 & 0.49 $\pm$ 0.13 & 0.49 $\pm$ 0.13 & 0.49 $\pm$ 0.13 & 0.57 $\pm$ 0.16 \\ 
\midrule
\end{tabular}
}

\caption{Results on \texttt{housing} dataset. Reported is the mean and standard deviation of performance over 125 trials. Grayed cells denote the best performer at that noise rate.}

\label{tbl:full-housing-matlab}
	
\end{table}

\begin{table}[t]
	\centering
	\renewcommand{\arraystretch}{1.25}

\subfloat[0-1 Error.]{
\begin{tabular}{@{}lllllll@{}}
\toprule
\toprule
& \textbf{Hinge} & \textbf{Logistic} & \textbf{Square} & \textbf{$t$-logistic} & \textbf{TanBoost} & \textbf{Unhinged} \\ 
\midrule
$\rhoPlus = 0$ & \cellcolor{gray!25}{0.01 $\pm$ 0.00} & 0.02 $\pm$ 0.00 & 0.03 $\pm$ 0.00 & 0.03 $\pm$ 0.00 & 0.02 $\pm$ 0.00 & 0.03 $\pm$ 0.00 \\ 
$\rhoPlus = 0.1$ & 0.05 $\pm$ 0.00 & 0.04 $\pm$ 0.01 & 0.04 $\pm$ 0.01 & \cellcolor{gray!25}{0.02 $\pm$ 0.01} & 0.04 $\pm$ 0.01 & 0.04 $\pm$ 0.01 \\ 
$\rhoPlus = 0.2$ & 0.05 $\pm$ 0.00 & 0.05 $\pm$ 0.01 & 0.05 $\pm$ 0.01 & \cellcolor{gray!25}{0.04 $\pm$ 0.01} & 0.05 $\pm$ 0.01 & 0.05 $\pm$ 0.01 \\ 
$\rhoPlus = 0.3$ & \cellcolor{gray!25}{0.05 $\pm$ 0.01} & 0.06 $\pm$ 0.01 & 0.06 $\pm$ 0.01 & 0.06 $\pm$ 0.02 & 0.06 $\pm$ 0.01 & 0.06 $\pm$ 0.01 \\ 
$\rhoPlus = 0.4$ & \cellcolor{gray!25}{0.06 $\pm$ 0.02} & 0.11 $\pm$ 0.06 & 0.11 $\pm$ 0.06 & 0.11 $\pm$ 0.06 & 0.11 $\pm$ 0.06 & 0.10 $\pm$ 0.05 \\ 
$\rhoPlus = 0.49$ & \cellcolor{gray!25}{0.33 $\pm$ 0.27} & 0.46 $\pm$ 0.16 & 0.46 $\pm$ 0.16 & 0.47 $\pm$ 0.16 & 0.47 $\pm$ 0.16 & 0.46 $\pm$ 0.16 \\ 
\midrule
\end{tabular}
}

\subfloat[1 - AUC.]{
\begin{tabular}{@{}lllllll@{}}
\toprule
\toprule
& \textbf{Hinge} & \textbf{Logistic} & \textbf{Square} & \textbf{$t$-logistic} & \textbf{TanBoost} & \textbf{Unhinged} \\ 
\midrule
$\rhoPlus = 0$ & \cellcolor{gray!25}{0.00 $\pm$ 0.00} & \cellcolor{gray!25}{0.00 $\pm$ 0.00} & 0.01 $\pm$ 0.00 & \cellcolor{gray!25}{0.00 $\pm$ 0.00} & 0.01 $\pm$ 0.00 & 0.02 $\pm$ 0.00 \\ 
$\rhoPlus = 0.1$ & 0.34 $\pm$ 0.18 & 0.03 $\pm$ 0.02 & 0.03 $\pm$ 0.02 & \cellcolor{gray!25}{0.00 $\pm$ 0.00} & 0.03 $\pm$ 0.02 & 0.04 $\pm$ 0.02 \\ 
$\rhoPlus = 0.2$ & 0.40 $\pm$ 0.17 & 0.07 $\pm$ 0.05 & 0.08 $\pm$ 0.05 & \cellcolor{gray!25}{0.04 $\pm$ 0.04} & 0.07 $\pm$ 0.05 & 0.08 $\pm$ 0.05 \\ 
$\rhoPlus = 0.3$ & 0.43 $\pm$ 0.17 & 0.17 $\pm$ 0.10 & 0.17 $\pm$ 0.10 & \cellcolor{gray!25}{0.14 $\pm$ 0.10} & 0.16 $\pm$ 0.10 & 0.16 $\pm$ 0.10 \\ 
$\rhoPlus = 0.4$ & 0.44 $\pm$ 0.18 & \cellcolor{gray!25}{0.30 $\pm$ 0.16} & \cellcolor{gray!25}{0.30 $\pm$ 0.16} & \cellcolor{gray!25}{0.30 $\pm$ 0.16} & \cellcolor{gray!25}{0.30 $\pm$ 0.16} & \cellcolor{gray!25}{0.30 $\pm$ 0.16} \\ 
$\rhoPlus = 0.49$ & 0.51 $\pm$ 0.19 & \cellcolor{gray!25}{0.46 $\pm$ 0.17} & \cellcolor{gray!25}{0.46 $\pm$ 0.17} & \cellcolor{gray!25}{0.46 $\pm$ 0.17} & \cellcolor{gray!25}{0.46 $\pm$ 0.17} & \cellcolor{gray!25}{0.46 $\pm$ 0.18} \\ 
\midrule
\end{tabular}
}

\caption{Results on \texttt{car} dataset. Reported is the mean and standard deviation of performance over 125 trials. Grayed cells denote the best performer at that noise rate.}

\label{tbl:full-car-matlab}

\end{table}

\begin{table}[t]
	\centering
	\renewcommand{\arraystretch}{1.25}

\subfloat[0-1 Error.]{
\begin{tabular}{@{}lllllll@{}}
\toprule
\toprule
& \textbf{Hinge} & \textbf{Logistic} & \textbf{Square} & \textbf{$t$-logistic} & \textbf{TanBoost} & \textbf{Unhinged} \\ 
\midrule
$\rhoPlus = 0$ & \cellcolor{gray!25}{0.00 $\pm$ 0.00} & \cellcolor{gray!25}{0.00 $\pm$ 0.00} & \cellcolor{gray!25}{0.00 $\pm$ 0.00} & \cellcolor{gray!25}{0.00 $\pm$ 0.00} & \cellcolor{gray!25}{0.00 $\pm$ 0.00} & \cellcolor{gray!25}{0.00 $\pm$ 0.00} \\ 
$\rhoPlus = 0.1$ & 0.10 $\pm$ 0.08 & 0.05 $\pm$ 0.01 & 0.01 $\pm$ 0.01 & 0.11 $\pm$ 0.02 & 0.02 $\pm$ 0.01 & \cellcolor{gray!25}{0.00 $\pm$ 0.00} \\ 
$\rhoPlus = 0.2$ & 0.19 $\pm$ 0.11 & 0.09 $\pm$ 0.02 & 0.05 $\pm$ 0.02 & 0.15 $\pm$ 0.02 & 0.06 $\pm$ 0.02 & \cellcolor{gray!25}{0.00 $\pm$ 0.00} \\ 
$\rhoPlus = 0.3$ & 0.31 $\pm$ 0.13 & 0.17 $\pm$ 0.03 & 0.14 $\pm$ 0.02 & 0.22 $\pm$ 0.03 & 0.16 $\pm$ 0.03 & \cellcolor{gray!25}{0.01 $\pm$ 0.00} \\ 
$\rhoPlus = 0.4$ & 0.39 $\pm$ 0.13 & 0.31 $\pm$ 0.04 & 0.30 $\pm$ 0.04 & 0.33 $\pm$ 0.04 & 0.31 $\pm$ 0.04 & \cellcolor{gray!25}{0.02 $\pm$ 0.02} \\ 
$\rhoPlus = 0.49$ & 0.50 $\pm$ 0.16 & 0.48 $\pm$ 0.04 & 0.47 $\pm$ 0.04 & 0.48 $\pm$ 0.04 & 0.48 $\pm$ 0.04 & \cellcolor{gray!25}{0.34 $\pm$ 0.21} \\ 
\midrule
\end{tabular}
}

\subfloat[1 - AUC.]{
\begin{tabular}{@{}lllllll@{}}
\toprule
\toprule
& \textbf{Hinge} & \textbf{Logistic} & \textbf{Square} & \textbf{$t$-logistic} & \textbf{TanBoost} & \textbf{Unhinged} \\ 
\midrule
$\rhoPlus = 0$ & \cellcolor{gray!25}{0.00 $\pm$ 0.00} & \cellcolor{gray!25}{0.00 $\pm$ 0.00} & \cellcolor{gray!25}{0.00 $\pm$ 0.00} & \cellcolor{gray!25}{0.00 $\pm$ 0.00} & \cellcolor{gray!25}{0.00 $\pm$ 0.00} & \cellcolor{gray!25}{0.00 $\pm$ 0.00} \\ 
$\rhoPlus = 0.1$ & 0.05 $\pm$ 0.06 & 0.01 $\pm$ 0.00 & \cellcolor{gray!25}{0.00 $\pm$ 0.00} & 0.05 $\pm$ 0.01 & \cellcolor{gray!25}{0.00 $\pm$ 0.00} & \cellcolor{gray!25}{0.00 $\pm$ 0.00} \\ 
$\rhoPlus = 0.2$ & 0.12 $\pm$ 0.11 & 0.03 $\pm$ 0.01 & 0.01 $\pm$ 0.00 & 0.07 $\pm$ 0.01 & 0.02 $\pm$ 0.01 & \cellcolor{gray!25}{0.00 $\pm$ 0.00} \\ 
$\rhoPlus = 0.3$ & 0.26 $\pm$ 0.18 & 0.10 $\pm$ 0.02 & 0.07 $\pm$ 0.02 & 0.14 $\pm$ 0.03 & 0.08 $\pm$ 0.02 & \cellcolor{gray!25}{0.00 $\pm$ 0.00} \\ 
$\rhoPlus = 0.4$ & 0.37 $\pm$ 0.19 & 0.25 $\pm$ 0.04 & 0.24 $\pm$ 0.04 & 0.27 $\pm$ 0.04 & 0.24 $\pm$ 0.04 & \cellcolor{gray!25}{0.00 $\pm$ 0.00} \\ 
$\rhoPlus = 0.49$ & 0.51 $\pm$ 0.23 & 0.47 $\pm$ 0.05 & 0.46 $\pm$ 0.05 & 0.47 $\pm$ 0.05 & 0.47 $\pm$ 0.05 & \cellcolor{gray!25}{0.25 $\pm$ 0.29} \\ 
\midrule
\end{tabular}
}

\caption{Results on \texttt{usps\_0\_vs\_7} dataset. Reported is the mean and standard deviation of performance over 125 trials. Grayed cells denote the best performer at that noise rate.}
\label{tbl:full-usps_0_vs_7-matlab}

\end{table}

\begin{table}[t]
	\centering
	\renewcommand{\arraystretch}{1.25}

\subfloat[0-1 Error.]{
\begin{tabular}{@{}lllllll@{}}
\toprule
\toprule
& \textbf{Hinge} & \textbf{Logistic} & \textbf{Square} & \textbf{$t$-logistic} & \textbf{TanBoost} & \textbf{Unhinged} \\ 
\midrule
$\rhoPlus = 0$ & 0.05 $\pm$ 0.00 & 0.04 $\pm$ 0.00 & \cellcolor{gray!25}{0.02 $\pm$ 0.00} & 0.04 $\pm$ 0.00 & \cellcolor{gray!25}{0.02 $\pm$ 0.00} & 0.19 $\pm$ 0.00 \\ 
$\rhoPlus = 0.1$ & 0.15 $\pm$ 0.03 & 0.05 $\pm$ 0.01 & \cellcolor{gray!25}{0.04 $\pm$ 0.01} & 0.24 $\pm$ 0.00 & \cellcolor{gray!25}{0.04 $\pm$ 0.01} & 0.19 $\pm$ 0.01 \\ 
$\rhoPlus = 0.2$ & 0.21 $\pm$ 0.03 & 0.08 $\pm$ 0.01 & \cellcolor{gray!25}{0.07 $\pm$ 0.01} & 0.24 $\pm$ 0.00 & \cellcolor{gray!25}{0.07 $\pm$ 0.01} & 0.19 $\pm$ 0.01 \\ 
$\rhoPlus = 0.3$ & 0.25 $\pm$ 0.03 & \cellcolor{gray!25}{0.14 $\pm$ 0.02} & \cellcolor{gray!25}{0.14 $\pm$ 0.02} & 0.24 $\pm$ 0.00 & \cellcolor{gray!25}{0.14 $\pm$ 0.02} & 0.19 $\pm$ 0.03 \\ 
$\rhoPlus = 0.4$ & 0.31 $\pm$ 0.05 & 0.28 $\pm$ 0.05 & 0.28 $\pm$ 0.04 & 0.24 $\pm$ 0.00 & 0.28 $\pm$ 0.04 & \cellcolor{gray!25}{0.22 $\pm$ 0.05} \\ 
$\rhoPlus = 0.49$ & 0.48 $\pm$ 0.09 & 0.47 $\pm$ 0.06 & 0.48 $\pm$ 0.05 & \cellcolor{gray!25}{0.40 $\pm$ 0.24} & 0.48 $\pm$ 0.05 & 0.45 $\pm$ 0.08 \\ 
\midrule
\end{tabular}
}

\subfloat[1 - AUC.]{
\begin{tabular}{@{}lllllll@{}}
\toprule
\toprule
& \textbf{Hinge} & \textbf{Logistic} & \textbf{Square} & \textbf{$t$-logistic} & \textbf{TanBoost} & \textbf{Unhinged} \\ 
\midrule
$\rhoPlus = 0$ & 0.01 $\pm$ 0.00 & 0.01 $\pm$ 0.00 & \cellcolor{gray!25}{0.00 $\pm$ 0.00} & 0.01 $\pm$ 0.00 & \cellcolor{gray!25}{0.00 $\pm$ 0.00} & 0.09 $\pm$ 0.00 \\ 
$\rhoPlus = 0.1$ & 0.10 $\pm$ 0.03 & \cellcolor{gray!25}{0.01 $\pm$ 0.00} & \cellcolor{gray!25}{0.01 $\pm$ 0.00} & 0.03 $\pm$ 0.01 & \cellcolor{gray!25}{0.01 $\pm$ 0.00} & 0.09 $\pm$ 0.01 \\ 
$\rhoPlus = 0.2$ & 0.20 $\pm$ 0.05 & 0.03 $\pm$ 0.01 & \cellcolor{gray!25}{0.02 $\pm$ 0.01} & 0.04 $\pm$ 0.01 & \cellcolor{gray!25}{0.02 $\pm$ 0.01} & 0.10 $\pm$ 0.02 \\ 
$\rhoPlus = 0.3$ & 0.30 $\pm$ 0.06 & 0.08 $\pm$ 0.02 & 0.08 $\pm$ 0.02 & 0.09 $\pm$ 0.02 & \cellcolor{gray!25}{0.07 $\pm$ 0.02} & 0.11 $\pm$ 0.03 \\ 
$\rhoPlus = 0.4$ & 0.40 $\pm$ 0.07 & 0.22 $\pm$ 0.04 & 0.22 $\pm$ 0.04 & 0.23 $\pm$ 0.04 & 0.22 $\pm$ 0.04 & \cellcolor{gray!25}{0.16 $\pm$ 0.07} \\ 
$\rhoPlus = 0.49$ & 0.49 $\pm$ 0.08 & 0.46 $\pm$ 0.05 & 0.46 $\pm$ 0.05 & 0.46 $\pm$ 0.05 & 0.45 $\pm$ 0.05 & \cellcolor{gray!25}{0.42 $\pm$ 0.15} \\ 
\midrule
\end{tabular}
}

\caption{Results on \texttt{splice} dataset. Reported is the mean and standard deviation of performance over 125 trials. Grayed cells denote the best performer at that noise rate.}

\label{tbl:full-splice-matlab}	
	
\end{table}

\begin{table}[t]
	\centering
	\renewcommand{\arraystretch}{1.25}
	
\subfloat[0-1 Error.]{
\begin{tabular}{@{}lllllll@{}}
\toprule
\toprule
& \textbf{Hinge} & \textbf{Logistic} & \textbf{Square} & \textbf{$t$-logistic} & \textbf{TanBoost} & \textbf{Unhinged} \\ 
\midrule
$\rhoPlus = 0$ & 0.16 $\pm$ 0.01 & \cellcolor{gray!25}{0.08 $\pm$ 0.00} & 0.10 $\pm$ 0.00 & 0.24 $\pm$ 0.00 & 0.09 $\pm$ 0.00 & 0.15 $\pm$ 0.00 \\ 
$\rhoPlus = 0.1$ & 0.14 $\pm$ 0.03 & \cellcolor{gray!25}{0.10 $\pm$ 0.02} & \cellcolor{gray!25}{0.10 $\pm$ 0.01} & 0.13 $\pm$ 0.06 & \cellcolor{gray!25}{0.10 $\pm$ 0.01} & 0.14 $\pm$ 0.01 \\ 
$\rhoPlus = 0.2$ & 0.17 $\pm$ 0.03 & \cellcolor{gray!25}{0.11 $\pm$ 0.02} & \cellcolor{gray!25}{0.11 $\pm$ 0.01} & 0.13 $\pm$ 0.05 & \cellcolor{gray!25}{0.11 $\pm$ 0.01} & 0.14 $\pm$ 0.01 \\ 
$\rhoPlus = 0.3$ & 0.23 $\pm$ 0.05 & 0.13 $\pm$ 0.02 & \cellcolor{gray!25}{0.12 $\pm$ 0.01} & 0.14 $\pm$ 0.04 & 0.13 $\pm$ 0.02 & 0.15 $\pm$ 0.01 \\ 
$\rhoPlus = 0.4$ & 0.33 $\pm$ 0.07 & 0.20 $\pm$ 0.04 & 0.19 $\pm$ 0.03 & 0.21 $\pm$ 0.04 & 0.19 $\pm$ 0.03 & \cellcolor{gray!25}{0.17 $\pm$ 0.03} \\ 
$\rhoPlus = 0.49$ & 0.49 $\pm$ 0.10 & 0.45 $\pm$ 0.07 & 0.44 $\pm$ 0.07 & 0.45 $\pm$ 0.07 & 0.45 $\pm$ 0.07 & \cellcolor{gray!25}{0.43 $\pm$ 0.12} \\ 
\midrule
\end{tabular}
}

\subfloat[1 - AUC.]{
\begin{tabular}{@{}lllllll@{}}
\toprule
\toprule
& \textbf{Hinge} & \textbf{Logistic} & \textbf{Square} & \textbf{$t$-logistic} & \textbf{TanBoost} & \textbf{Unhinged} \\ 
\midrule
$\rhoPlus = 0$ & 0.03 $\pm$ 0.00 & \cellcolor{gray!25}{0.02 $\pm$ 0.00} & 0.05 $\pm$ 0.00 & \cellcolor{gray!25}{0.02 $\pm$ 0.00} & 0.04 $\pm$ 0.00 & 0.07 $\pm$ 0.00 \\ 
$\rhoPlus = 0.1$ & 0.06 $\pm$ 0.01 & 0.04 $\pm$ 0.00 & 0.05 $\pm$ 0.00 & \cellcolor{gray!25}{0.03 $\pm$ 0.00} & 0.04 $\pm$ 0.00 & 0.07 $\pm$ 0.00 \\ 
$\rhoPlus = 0.2$ & 0.10 $\pm$ 0.03 & 0.05 $\pm$ 0.00 & 0.05 $\pm$ 0.00 & \cellcolor{gray!25}{0.04 $\pm$ 0.00} & 0.05 $\pm$ 0.00 & 0.07 $\pm$ 0.00 \\ 
$\rhoPlus = 0.3$ & 0.17 $\pm$ 0.06 & \cellcolor{gray!25}{0.06 $\pm$ 0.01} & \cellcolor{gray!25}{0.06 $\pm$ 0.01} & \cellcolor{gray!25}{0.06 $\pm$ 0.01} & \cellcolor{gray!25}{0.06 $\pm$ 0.01} & 0.07 $\pm$ 0.01 \\ 
$\rhoPlus = 0.4$ & 0.32 $\pm$ 0.12 & 0.12 $\pm$ 0.02 & 0.12 $\pm$ 0.02 & 0.12 $\pm$ 0.02 & 0.12 $\pm$ 0.02 & \cellcolor{gray!25}{0.09 $\pm$ 0.02} \\ 
$\rhoPlus = 0.49$ & 0.49 $\pm$ 0.14 & 0.43 $\pm$ 0.08 & 0.43 $\pm$ 0.08 & 0.43 $\pm$ 0.07 & 0.43 $\pm$ 0.08 & \cellcolor{gray!25}{0.39 $\pm$ 0.19} \\ 
\midrule
\end{tabular}
}

\caption{Results on \texttt{spambase} dataset. Reported is the mean and standard deviation of performance over 125 trials. Grayed cells denote the best performer at that noise rate.}

\label{tbl:full-spambase-matlab}

\end{table}

\clearpage

\bibliography{references}
\bibliographystyle{plainnat}

\end{document}